\documentclass{article}

     \usepackage[final]{neurips_2020}

\usepackage{algorithm,algorithmic}
\usepackage[utf8]{inputenc} 
\usepackage[T1]{fontenc}    
\usepackage{hyperref}       
\usepackage{url}            
\usepackage{booktabs}       
\usepackage{amsfonts}       
\usepackage{nicefrac}       
\usepackage{microtype}      

\usepackage{parskip,verbatim}

\usepackage{booktabs} 

\usepackage[utf8]{inputenc}

\usepackage{amsthm}
\usepackage{thm-restate}

\usepackage{mathtools}

\usepackage{amsmath}
\usepackage{bbm}

\usepackage{amsfonts}

\usepackage{amssymb}
\usepackage{float}

\usepackage{microtype}

\usepackage{geometry}

\usepackage{graphpap,amscd,mathrsfs,graphicx,lscape,enumitem,dsfont,bm,url,subfigure}
\usepackage{epsfig,amstext,xspace}

\usepackage{thmtools,thm-restate}
\usepackage[utf8]{inputenc} 
\usepackage[T1]{fontenc}    
\usepackage{hyperref}       
\usepackage{url}            
\usepackage{booktabs}       
\usepackage{amsfonts}       
\usepackage{nicefrac}       
\usepackage{microtype}      
\usepackage{bbm}
\usepackage{cleveref}
\usepackage{natbib}

\newif\ifpx
\ifpx

  \usepackage[varg]{pxfonts}
\fi

\DeclareMathOperator*{\En}{\mathbb{E}}

 \newcommand{\R}{\mathbb{R}}
 
\newcommand{\bigO}{\mathcal{O}}

\newcommand{\opt}{\text{\textsc{Opt}} }

{
 \theoremstyle{plain}
      
}
\newtheorem{nono-theorem}{Theorem}[]
\theoremstyle{plain}
\newtheorem{theorem}{Theorem}[section]

\newtheorem{lemma}[theorem]{Lemma}

\newtheorem{remark}[theorem]{Remark}

\newtheorem{prop}[theorem]{Proposition}
\theoremstyle{definition}
\newtheorem{definition}[theorem]{Definition}

\newcommand{\Exp}{\mathbb{E}}

\newcommand{\states}{\mathcal{S}}
\newcommand{\actions}{\mathcal{A}}
\newcommand{\sa}{(s,a)}

\newcommand{\boldpi}{{\pi}}

\newcommand{\bonus}{b
}

\newcommand{\kh}{_{k;h}}

\newcommand{\truemodel}{\mathcal{M}^{\star}}
\newcommand{\model}{\mathcal{M}}
\newcommand{\resources}{\mathcal{D}}

\newcommand{\vs}{{s}}
\newcommand{\va}{a}

\newcommand{\boldpist}{\boldpi^{\star}}

\newcommand{\cmdp}{\textsc{cMDP}}

\newcommand{\bellman}{\textsc{Bell}}

\newcommand{\xhdr}[1]{\vspace{1mm} \noindent{\bf #1}}

\usepackage{color}              
\definecolor{cerulean}{rgb}{0.10, 0.58, 0.75}

\usepackage[suppress]{color-edits}
\addauthor{tl}{cyan}
\addauthor{ws}{red}
\addauthor{ms}{purple}
\addauthor{as}{magenta}
\addauthor{sm}{green}
\addauthor{kb}{blue}
\addauthor{md}{cerulean}

\newcommand{\added}[1]{\color{red}#1\color{black}}

\title{
Constrained episodic reinforcement learning in concave-convex and knapsack settings
}

\author{%
  Kiant\'e Brantley\\
  \textrm{University of Maryland}\\
  \texttt{kdbrant@cs.umd.edu}\\
  \And
  Miroslav Dud\'ik\\
  \textrm{Microsoft Research} \\
  \texttt{mdudik@microsoft.com} \\
  \And
  Thodoris Lykouris\\
  \textrm{Microsoft Research}\\
  \texttt{thlykour@microsoft.com}\\
  \And
  Sobhan Miryoosefi\\
  \textrm{Princeton University} \\
  \texttt{miryoosefi@cs.princeton.edu} \\
  \And
  Max Simchowitz\\
  \textrm{UC Berkeley}\\
  \texttt{msimchow@berkeley.edu}\\
  \And
  Aleksandrs Slivkins\\
  \textrm{Microsoft Research}\\
  \texttt{slivkins@microsoft.com}\\
  \And
  Wen Sun\\
  \textrm{Cornell University}\\
  \texttt{ws455@cornell.edu}\\
}

\begin{document}


\date{}
\maketitle
 \begin{abstract}

\mdedit{We propose an algorithm for tabular episodic reinforcement learning (RL) with constraints. We provide a modular analysis with strong theoretical guarantees for two general settings. First is the convex-concave setting: maximization of a concave reward function subject to constraints that expected values of some vector quantities (such as the use of unsafe actions) lie in a convex set. Second is the knapsack setting: \wsedit{maximization of reward subject to the constraint that the total consumption of any of the specified resources does not exceed specified levels during the whole learning process}. Previous work in constrained RL is limited to linear expectation constraints (a special case of convex-concave setting), or focuses on feasibility question, or on single-episode settings. Our experiments demonstrate that the proposed algorithm significantly outperforms these approaches in constrained episodic benchmarks.}\looseness=-1


 \end{abstract}




\section{Introduction}
\label{sec:intro}
Standard reinforcement learning (RL) approaches seek to maximize a scalar reward~\citep{SuttonBa98, SuttonBa18, Schulman2015TRPO, mnih2015human},
but in many settings this is insufficient, because the desired properties of the agent behavior are better described using constraints. For example, an autonomous
vehicle should not only get to the destination, but should also respect safety, fuel efficiency, and human comfort constraints along the way~\citep{Hoang2019BPLUC}; a robot should not only fulfill its task, but should also control its wear and tear, for example, by limiting the torque exerted on
its motors~\citep{tessler2018reward}.
Moreover, in many settings, we wish to satisfy such constraints already during \emph{training} and not only
during the \emph{deployment}. For example, a power grid, an autonomous vehicle, or a real robotic hardware should
avoid costly failures, where the hardware is damaged or humans are harmed, already during training~\citep{leike2017ai,Ray2019}. Constraints are
also key in additional sequential decision\tledit{-}making applications, such as
dynamic pricing with limited supply (\mdedit{e.g.}, \citealp{BZ09,DynPricing-ec12}), scheduling of resources on a computer cluster~\citep{Mao2016RLSystems},
and imitation learning, where the goal is to stay close to an expert behavior \citep{SyedSchapire2008,ziebart2008maximum,sun2019provably}.

In this paper we study \emph{constrained episodic reinforcement learning}, which encompasses all of these applications. An important characteristic of our approach, distinguishing it from previous work~(\mdedit{e.g.}, \citealp{altman-constrainedMDP,achiam2017constrained,tessler2018reward,MiryoosefiBrDaDuSc19,Ray2019}), is our focus on \emph{efficient exploration}, leading to reduced sample complexity. Notably, the modularity of our approach enables extensions to more complex settings such as (i) maximizing concave objectives under convex constraints, and (ii) reinforcement learning under hard constraints, where the learner has to stop when some constraint is violated (e.g., a car runs out of gas). For these extensions, which we refer to as \textit{concave-convex setting} and \textit{knapsack setting}, we provide the first regret guarantees in the episodic setting (see related work below for a detailed comparison). Moreover, our guarantees are \emph{anytime}, meaning that the constraint violations are bounded at any point during learning, even if the learning process is interrupted. This is important for those applications where the system continues to learn after it is deployed.

Our approach relies on the principle of \emph{optimism under uncertainty} to efficiently explore.
Our learning algorithms optimize their actions with respect to a model based on the empirical statistics, while optimistically overestimating rewards and underestimating the resource consumption (i.e., overestimating the distance from the constraint). This idea was previously introduced in multi-armed bandits \citep{AgrawalDevanurEC14}; extending it to episodic reinforcement learning poses additional challenges since the policy space is exponential in the episode horizon. Circumventing these challenges, we provide a modular way to analyze this approach in the basic setting where both rewards and constraints are linear (Section~\ref{sec:basic_setting}) and then transfer this result to the more complicated concave-convex and knapsack settings (Sections~\ref{sec:concave-convex}~and~\ref{sec:knapsacks}). We empirically compare our approach with the only previous works that can handle
convex constraints and show that our algorithmic innovations lead to significant empirical improvements (Section \ref{sec:empirical}).

\paragraph{Related work.} Sample-efficient exploration in constrained episodic reinforcement learning has only recently started to receive attention. Most previous works on episodic reinforcement learning focus on unconstrained settings \citep{jaksch2010near,AzarOsMu17,dann2017unifying}. A notable exception is the work of \cite{Cheung19} \wsedit{and \cite{tarbouriech2019active}}. \wsedit{Both of these works consider vectorial feedback and aggregate reward functions}, and provide theoretical guarantees for the reinforcement learning setting with a single episode, but require a strong reachability \wsedit{or communication} assumption, which is not needed in the episodic setting studied here. Also, compared to \cite{Cheung19},  our results for the knapsack setting allow for a significantly smaller budget, as we illustrate in Section~\ref{sec:knapsacks}. Moreover, our approach is based on a tighter bonus, which leads to a superior empirical performance (see Section~\ref{sec:empirical}). Recently, there have also been several concurrent and independent works on sample-efficient exploration for reinforcement learning with constraints  \citep{singh2020learning,efroni2020exploration,qiu2020upper,ding2020provably,zheng2020constrained}. Unlike our work, all of these approaches focus on linear reward objective and linear constraints and do not handle the concave-convex and knapsack settings that we consider.

Constrained reinforcement learning has also been studied in settings that do not focus on sample-efficient exploration~\citep{achiam2017constrained,tessler2018reward,MiryoosefiBrDaDuSc19}. Among these, only \cite{MiryoosefiBrDaDuSc19} handle convex constraints, albeit without a reward objective (they solve the feasibility problem). Since these works do not focus on sample-efficient exploration, their performance drastically deteriorates when the task requires exploration (as we show in Section~\ref{sec:empirical}).

Sample-efficient exploration under constraints has been studied in multi-armed bandits, starting \mdedit{with} a line of work on dynamic pricing with limited supply  \citep{BZ09,BesbesZeevi-OR11,DynPricing-ec12,Wang-OR14}.
A general setting for bandits with global knapsack constraints
(\emph{bandits with knapsacks}) was defined and solved by \mdedit{\citet{BwK-focs13} (see also Ch.~10 of \citealp{slivkins-MABbook})}.
Within this literature, the closest to ours is the work of \cite{AgrawalDevanurEC14}, who study bandits with concave objectives and convex constraints. Our work is directly inspired by theirs and lifts their techniques to the more general episodic reinforcement learning setting.

\section{Model and preliminaries}
\label{sec:model}


In episodic reinforcement learning, a learner repeatedly interacts with an environment across $K$ episodes. The environment includes the state space $\states$, the action space $\actions$, the episode horizon $H$, and the {initial state $s_0$}.%
\footnote{A fixed and known initial state is without loss of generality. In general, there is a fixed but unknown distribution $\rho$ from which the initial state is drawn before each episode. We modify the MDP by adding a new state {$s_{0}$} as initial state, such that the next state is sampled from $\rho$ for any action. Then $\rho$ is ``included'' within the transition probabilities. The extra state $s_0$ does not contribute any reward and does not consume any resources.}
To capture constrained settings, the environment includes a set $\resources$ of $d$ resources where each $i\in\resources$ has a capacity constraint $\xi(i)\in\R^+$. The above are fixed and known to the learner.

\xhdr{Constrained Markov \mdedit{decision process}.} We work with MDPs that have resource consumption in addition to rewards. Formally, a \emph{constrained} MDP ($\cmdp$) is a triple $\model=(p,r,\bm{c})$ that describes transition probabilities $p:\states\times\actions\rightarrow \Delta(\states)$, rewards $r:\states\times\actions\rightarrow [0,1]$, and resource consumption $\bm{c}:\states\times\actions\rightarrow [0,1]^d$. For convenience, we denote $c(s,a,i) = c_i(s,a)$. We allow stochastic rewards and consumptions, in which case $r$ and $\bm{c}$ refer to the conditional expectations, conditioned on $s$ and $a$ (our definitions and algorithms are based on this conditional expectation rather than the full conditional distribution).

\mdedit{We use the above definition to describe two kinds of $\cmdp$s. The} \emph{true} $\cmdp$ $\truemodel=(p^{\star},r^{\star},\bm{c}^{\star})$ is fixed but \emph{unknown} to the learner.  Selecting action $a$ at state $s$ results in rewards and consumptions drawn from (possibly correlated) distributions with means $r^{\star}(s,a)$ and $\bm{c}^{\star}(s,a)$ and supports in $[0,1]$ and $[0,1]^d$ respectively. Next states are generated from transition probabilities $p^{\star}(s,a)$.
\mdedit{The second kind of $\cmdp$ arises in our algorithm, which is model-based and at episode $k$ uses a $\cmdp$ $\model^{(k)}$.}

\xhdr{Episodic reinforcement learning protocol.}  At episode $k\in[K]$, the learner commits to a policy $\boldpi_k=(\boldpi_{k,h})_{h=1}^H$ where $\boldpi_{k,h}: \states \rightarrow \Delta(\actions)$ specifies how to select actions at step $h$ for every state. The learner starts from state {$\vs_{k,1} = s_0$}. At step $h=1,\ldots,H$, she selects an action $a_{k,h}\sim \pi_{k,h}(s_{k,h})$. The learner earns reward $r_{k,h}$ and suffers consumption $\bm{c}_{k,h}$, both drawn from the true $\cmdp$ $\truemodel$ on state-action pair $(s_{k,h},a_{k,h})$ as described above, and transitions to state $s_{k,h+1}\sim p^{\star}(s_{k,h},a_{k,h})$.

\xhdr{Objectives.} In the basic setting (Section~\ref{sec:basic_setting}), the learner wishes to maximize reward while respecting the consumption constraints in expectation by competing favorably against the following benchmark:
\begin{equation}
    \max_{\boldpi} \Exp^{\boldpi, p^{\star}}\Big[\sum_{h=1}^H r^{\star}\big(\vs_h,\va_h\big)\Big]
    \qquad
    \text{s.t.}
    \qquad
    \forall i\in\resources: \Exp^{\boldpi, p^{\star}}\Big[\sum_{h=1}^H c^{\star}\big(\vs_h,\va_h, i\big)\Big]\leq \xi(i),\label{eq:objective}
\end{equation}
where $\Exp^{\boldpi, p}$ denotes the expectation over the run of policy $\boldpi$ according to transitions $p$, and $\vs_h,\va_h$ are the induced random state-action pairs. We denote by $\boldpist$ the policy that maximizes this objective.

For the basic setting, we track two performance measures: \emph{reward regret} compares the learner's total reward to the benchmark and \emph{consumption regret} bounds excess in resource consumption:
\begin{align}
\label{eq:regret_def}
    &\textsc{RewReg}(k)\coloneqq \Exp^{\boldpist, p^{\star}}\Big[\sum_{h=1}^{H}r^{\star}\big(\vs_h,\va_h\big)\Big] -\frac{1}{k}\sum_{t=1}^{k}\Exp^{\boldpi_t, p^{\star}}\Big[\sum_{h=1}^{H}r^{\star}\big(\vs_h,\va_h\big)\Big], \\
    &\textsc{ConsReg}(k)\coloneqq \max_{i\in \resources} \Big(\frac{1}{k}\sum_{t=1}^{k}\Exp^{\boldpi_t, p^{\star}}\Big[\sum_{h=1}^{H}c^{\star}\big(\vs_h,\va_h,i\big)\Big]-\xi(i)\Big). \label{eq:regret_def_constraint}
\end{align}
Our guarantees are \emph{anytime}, i.e., they hold at any episode $k$ and not only after the last episode.

We also consider two extensions. In Section~\ref{sec:concave-convex}, we consider a concave reward objective and convex consumption constraints. In Section~\ref{sec:knapsacks}, we require consumption constraints to be satisfied \mdedit{with high probability under a cumulative budget across all $K$ episodes, rather than in expectation in a single episode.}\looseness=-1


\xhdr{Tabular MDPs.} We assume that the state space $\states$ and the action space $\actions$ are finite (tabular setting). We construct standard empirical estimates separately for each state-action pair $(s,a)$, using the learner's observations up to and not including a given episode $k$. Eqs.~(\ref{eq:empirical-N}--\ref{eq:empirical-c}) define sample counts, empirical transition probabilities, empirical rewards, and empirical resource consumption.%
\footnote{The $\max$ operator in Eq. \eqref{eq:empirical-N} is to avoid dividing by $0$.}
\begin{align}
         N_{k}(s,a)
            &=\max\bigg\{1,\;\sum_{t\in[k-1],\,h\in[H]} \mathbf{1}\{\vs_{t,h}=s, \va_{t,h}=a\}\bigg\},
            \label{eq:empirical-N} \\
         \widehat{p}_k(s'|s,a)
            &=\frac{1}{N_k(s,a)}\;
            \sum_{t\in[k-1],\,h\in[H]} \mathbf{1}\{\vs_{t,h}=s,\va_{t,h}=a,\vs_{t,h+1}=s'\},
            \label{eq:empirical-p} \\
         \widehat{r}_k(s,a)
          &= \frac{1}{N_k(s,a)}\;
            \sum_{t\in[k-1],\,h\in[H]}  r_{t,h}\cdot \mathbf{1}\{\vs_{t,h}=s,\va_{t,h}=a\},
            \label{eq:empirical-r} \\
          \widehat{c}_k(s,a,i)
          &= \frac{1}{N_k(s,a)}\;
            \sum_{t\in[k-1],\,h\in[H]}  c_{t,h,i}\cdot \mathbf{1}\{\vs_{t,h}=s,\va_{t,h}=a\} \quad \forall i\in\resources.
            \label{eq:empirical-c}
         \end{align}

\xhdr{Preliminaries for theoretical analysis.}
\mdedit{The \emph{$Q$-function}}
is a standard object in RL that tracks the learner's expected performance if she starts from state $s\in\states$ at step $h$, selects action $a\in\actions$, and then follows a policy $\boldpi$ under a model with transitions $p$ for the remainder of the episode. We parameterize it by the \emph{objective function} $m:\states\times\actions \to [0,1]$, which can be either a reward, i.e., $m(s,a)=r(s,a)$, or consumption of some resource $i\in\resources$, i.e., $m(s,a)=c(s,a,i)$. (For the unconstrained setting, the objective is the reward.)
The performance of the policy in a particular step $h$ is evaluated by the value function $V$ which corresponds to the expected $Q$-function of the selected action (where the expectation is taken over the possibly randomized action selection of $\boldpi$). \mdedit{The $Q$ and value functions can be both} recursively defined by dynamic programming:
\begin{align*}Q_m^{\boldpi,p}(s,a,h)&=m(s,a)+\sum_{s'\in\states}p(s'|s,a)V_{m}^{\boldpi,p}(s',h+1),
\\
V_m^{\boldpi,p}(s,h)&=\Exp_{a\sim \boldpi(\cdot |s)}\Big[Q_m^{\boldpi,p}(s,a,h)\Big]\quad \text{and} \quad V_{m}^{\boldpi, p}(s,H+1)=0.
\end{align*}
By slight abuse of notation, for $m\in\{r\}\cup \{c_i\}_{i\in\resources}$, we denote by $m^{\star}\in\{r^{\star}\}\cup \{c_i^{\star}\}_{i\in\resources}$ the corresponding objectives with respect to the rewards and consumptions of the true $\cmdp$ $\truemodel$. For objectives $m^{\star}$ and transitions $p^{\star}$, the above are the \emph{Bellman equations} of the system \citep{Bellman1957}.

Estimating the $Q$-function based on the model parameters $p$ and $m$ rather than the ground truth parameters $p^{\star}$ and $m^{\star}$ introduces errors. These errors are localized across stages by the notion of \emph{Bellman error} which contrasts the performance of policy $\boldpi$ starting from stage $h$ under the model parameters to a benchmark that behaves according to the model parameters starting from the next stage $h+1$ but uses the true parameters of the system in stage $h$. More formally, for objective $m$:
\begin{align}\label{eq:Bellman-errors-defn}
&\bellman_{m}^{\boldpi,p}(s,a,h)=
Q_{m}^{\boldpi,p}(s,a,h)-\Big(m^{\star}(s,a)
+\sum_{s'\in\states}p^{\star}(s'|s,a)V_{m}^{\boldpi,p}(s',h+1)\Big).
\end{align}
Note that when the $\cmdp$ is $\truemodel$ ($m=m^{\star}$, $p=p^{\star}$), there is no mismatch and $\bellman_{m^{\star}}^{\boldpi,p^{\star}}=0$. 

\section{Warm-up algorithm and analysis in the basic setting}
\label{sec:basic_setting}


In this section, we introduce a simple algorithm that allows to simultaneously bound reward and consumption regrets for the basic setting introduced in the previous section. Even in this basic setting, we provide the first sample-efficient guarantees in constrained episodic reinforcement learning.%
\footnote{We refer the reader to the related work (in Section \ref{sec:intro}) for discussion on concurrent and independent papers. Unlike our results, these papers do not extend to either concave-convex or knapsack settings.} The modular analysis of the guarantees also allows us to subsequently extend (in Sections~\ref{sec:concave-convex} and \ref{sec:knapsacks}) the algorithm and guarantees to the more general concave-convex and knapsack settings.

\xhdr{Our algorithm.} At episode $k$, we construct an estimated $\cmdp$ $\model^{(k)}=\big(p^{(k)},r^{(k)},\bm{c}^{(k)}\big)$ based on the observations collected so far. The estimates are \emph{bonus-enhanced} (formalized below) to encourage more targeted exploration. Our algorithm \textsc{ConRL} selects a policy $\boldpi_k$ by solving the following constrained optimization problem which we refer to as $\textsc{BasicConPlanner}(p^{(k)},r^{(k)},\bm{c}^{(k)})$:
\begin{align*}
    &\max_{\boldpi} \Exp^{\boldpi, p^{(k)}}\Big[\sum_{h=1}^H r^{(k)}\big(\vs_h,\va_h\big)\Big] \qquad \text{s.t.} \qquad &\forall i\in\resources: \Exp^{\boldpi, p^{(k)}}\Big[\sum_{h=1}^H c^{(k)}\big(\vs_h,\va_h, i\big)\Big]\leq \xi(i).
\end{align*}
The above optimization problem is similar to the objective \eqref{eq:objective} but uses the estimated model instead of the (unknown to the learner) true model. We also note that this optimization problem can be optimally solved as it is a linear program on the occupation measures \citep{puterman2014markov}, i.e., setting as variables the probability of each state-action pair and imposing flow conservation constraints with respect to the transitions. This program is described in Appendix~\ref{app:subsec_basicConPlanner}.

\xhdr{Bonus-enhanced model.}  A standard approach to implement the principle of optimism under uncertainty is to introduce, at each episode $k$, a \emph{bonus term} $\smash{\widehat{\bonus}_k(s,a)}\vphantom{r^{(k)}}$ that favors under-explored actions. Specifically, we add this bonus to the empirical rewards \eqref{eq:empirical-r}, and subtract it from the consumptions~\eqref{eq:empirical-c}:
    $r^{(k)}(s,a) = \widehat{r}_k(s,a)+\smash{\widehat{\bonus}_k(s,a)}$
  and
    $c^{(k)}(s,a,i) = \widehat{c}_k(s,a,i)-\smash{\widehat{\bonus}_k(s,a)}$
  for each resource $i$.

\renewcommand{\thefootnote}{\added{3.5}}
\wsedit{Similar to unconstrained analogues \citep{AzarOsMu17,dann2017unifying}, we define the bonus as:}%
\footnote[1]{\added{The NeurIPS 2020 version includes a small bug, leading to an incorrect dependence on~$H$ in Theorem~\ref{thm:tabular}. This version fixes it
by adjusting Eq.~\eqref{eq:bonus}, Theorem~\ref{thm:tabular} and the relevant proofs. Changes in the main text are noted in red.
Changes in the appendix are limited to Appendices~\ref{app:valid_bonus},~\ref{app:covering_bellman_error}, and~\ref{app:guarantee_basic} and the statement of Lemma~\ref{lemma:expectation_to_realization}.}}
\begin{align}\label{eq:bonus}
\widehat{\bonus}_k(s,a)=\added{\min\left\{2 H,} \; H\sqrt{\frac{2\ln\bigl(8SAH(d+1)k^2/\delta\bigr)}{N_k(s,a)}}\added{\right\}},
\end{align}
where $\delta>0$ is the desired failure probability of the algorithm and $N_k(s,a)$ is the number of times $(s,a)$ pair is visited, c.f.~\eqref{eq:empirical-N}, $S=|\states|$, and $A=|\actions|$. Thus, under-explored actions have a larger bonus, and therefore appear more appealing to the planner. For estimated transition probabilities, we just use the empirical averages    \eqref{eq:empirical-p}: $p^{(k)}(s'|s,a)=\widehat{p}(s'|s,a)$.
\renewcommand{\thefootnote}{\arabic{footnote}}

\newcommand{\Epi}{\Exp^{\boldpi}}

\xhdr{Valid bonus and Bellman-error decomposition.} For a bonus-enhanced model to achieve effective exploration, the resulting bonuses need to be \emph{valid}, i.e., they should ensure that the estimated rewards overestimate the true rewards
\mdedit{and the estimated consumptions underestimate the true consumptions}.

\begin{definition}\label{defn:valid_bonus}
A bonus $\bonus_k:\states\times\actions\rightarrow\R$ is valid if, $\forall s\in\states,a\in\actions, h\in[H], m\in\{r\}\cup\{c_i\}_{i\in\resources}$:
\begin{align*}
    &\Big|\Big(\widehat{m}_k(s,a)-m^{\star}(s,a)\Big)
    +\sum_{s'\in\states}\Big(\widehat{p}_k(s'|s,a)-p^{\star}(s'|s,a)\Big)V_{m^{\star}}^{\pi^{\star},p^{\star}}(s',h+1)
    \Big|
\leq \bonus_k(s,a).    \end{align*}
\end{definition}
By classical concentration bounds (Appendix \ref{app:valid_bonus}), the bonus $\widehat{\bonus}_k$ of Eq. \eqref{eq:bonus} satisfies this condition:
\begin{lemma}\label{lem:valid_bonus}
With probability $1-\delta$, the bonus $\widehat{\bonus}_k(s,a)$ is valid for all episodes $k$ simultaneously.
\end{lemma}

Our algorithm solves the \textsc{BasicConPlanner} optimization problem based on a bonus-enhanced model. When the bonuses are valid, we can upper bound the per-episode regret by the expected sum of Bellman errors across steps. This is the first part in classical unconstrained analyses and the following proposition extends this decomposition to constrained episodic reinforcement learning.
The proof uses the so-called simulation lemma \citep{Kearns2002} and is provided in Appendix~\ref{app:simulation_lemma}.

\begin{prop} \label{prop:regret_decomposition} If $\widehat{\bonus}_k(s,a)$ is valid for all episodes $k$ simultaneously then the per-episode reward and consumption regrets can be upper bounded by the expected sum of Bellman errors \eqref{eq:Bellman-errors-defn}:
\begin{align}
\Exp^{\boldpist, p^{\star}}\Big[\sum_{h=1}^Hr^{\star}\big(\vs_h,\va_h\big)\big]-\Exp^{\boldpi_k, p^{\star}}\Big[\sum_{h=1}^Hr^{\star}\big(\vs_h,\va_h\big)\Big]
&\leq \Exp^{\boldpi_k}\Big[\sum_{h=1}^H \Big|\bellman_{r^{(k)}}^{\boldpi_k,p^{(k)}}\big(\vs_h,\va_h,h\big)\Big|\Big]
\label{eq:reward_decomposition}
\\
\forall i\in\resources:\qquad \Exp^{\boldpi_k, p^{\star}}\Big[\sum_{h=1}^Hc^{\star}\big(\vs_h,\va_h,i\big)\Big]-\xi(i)
&\leq \Exp^{\boldpi_k}\Big[\sum_{h=1}^H \Big|\bellman_{c_i^{(k)}}^{\boldpi_k,p^{(k)}}\big(\vs_h,\va_h,h\big)\Big|\Big]. \label{eq:consumption_decomposition}
\end{align}
\end{prop}

\xhdr{Final guarantee.} One difficulty with directly bounding the Bellman error is that the value function is not independent of the draws forming $r^{(k)}(s,a)$, $\bm{c}^{(k)}(s,a)$, and $p^{(k)}(s'|s,a)$. Hence we cannot apply Hoeffding inequality directly. While  \cite{AzarOsMu17} propose a trick to get an \mdedit{$\bigO(\sqrt{S})$} bound on Bellman error in unconstrained settings, the trick relies on the crucial property of Bellman optimality: for an unconstrained MDP, its optimal policy $\pi^\star$ satisfies the condition, $\smash{V^{\pi^\star}_{r^{\star}}}(s,h) \geq V^{\pi}_{r^{\star}}(s,h)$ for all $s, h, \pi$ (i.e., $\pi^\star$ is optimal at any state).  However, when constraints exist, the optimal policy does not satisfy the Bellman optimality property. Indeed, we can only guarantee optimality with respect to the initial state distribution, i.e., $V^{\pi^\star}_{r^\star}(s_0,1) \geq V^{\pi}_{r^\star}(s_0,1)$ for any $\pi$, but not everywhere else. This illustrates a fundamental difference between constrained MDPs and unconstrained MDPs. Thus we cannot directly apply the trick from \cite{AzarOsMu17}.  Instead we follow an alternative approach of bounding the value function via an $\epsilon$-net over the possible values. This analysis leads to a guarantee that is weaker by a factor of $\sqrt{S}$ than the unconstrained results. The proof is provided in Appendix~\ref{app:guarantee_basic}.~\looseness=-1

\begin{theorem}\label{thm:tabular}
\mdedit{There exists an absolute constant $c\in\mathbb{R}^+$ such that,} with probability at least $1-3\delta$, reward and consumption regrets are both upper bounded by:
$$
\tfrac{c}{\sqrt{k}}\cdot \added{H^{2.5}} S \sqrt{A}\cdot\sqrt{\ln(k)\ln\big(SAH(d+1)k/\delta\big)}+\tfrac{c}{k}\cdot S^{3/2}A\added{H^3}\sqrt{\ln\big(2SAH(d+1)k/\delta\big)}.
$$
\end{theorem}

\xhdr{Comparison to single-episode results.} In single-episode setting, \cite{Cheung19} achieves $\smash{\sqrt{S}}$ dependency under the further assumption that the
transitions are sparse, i.e., $\| p^\star(s,a) \|_0 \ll S$ for all $\sa$. We do not make such assumptions on the sparsity of the MDP and we note that the regret bound of \cite{Cheung19} scales \mdedit{linearly in $S$} when $\| p^\star(s,a) \|_0 =\Theta( S)$.
\mdedit{Also, the single-episode setting requires a strong reachability assumption, not present in the episodic setting.}

\begin{remark}
\tledit{The aforementioned regret bound can be turned into a PAC bound of $\smash{\tilde{\mathcal{O}}\big(\frac{S^2A\added{H^5}}{\epsilon^2}\Big)}$ by taking the uniform mixture of policies $\pi_1,\pi_2,\ldots,\pi_k$.}
\end{remark}
\section{Concave-convex setting}
\label{sec:concave-convex}

We now extend the algorithm and guarantees derived for the basic setting to when the objective is concave function of the accumulated reward and the constraints \mdedit{are expressed as} a convex function of the cumulative consumptions. Our approach is modular, seamlessly building on the basic setting.

\xhdr{Setting and objective.} Formally, there is a concave reward-objective function $f:\R\rightarrow\R$ and a convex consumption-objective function $g:\R^d\rightarrow\R$; the only assumption is that these functions are $L$-Lipschitz for some constant $L$, i.e., $| f(x) - f(y)| \leq L |x - y|$ for any $x,y\in\mathbb{R}$, and $| g(x) - g(y)| \leq L\| x - y\|_1$ for any $x,y\in\mathbb{R}^d$. Analogous to
\eqref{eq:objective}, the learner wishes to compete against the following benchmark which can be viewed as a reinforcement learning variant of the benchmark used by \cite{AgrawalDevanurEC14} in multi-armed bandits:
\begin{align}\label{eq:objective_convex}
\max_{\boldpi}f\Big(\Exp^{\boldpi,p^{\star}}\Big[\sum_{h=1}^H r^{\star}\big(\vs_h,\va_h\big)\Big]\Big) \quad \text{s.t.} \quad g\Big(\Exp^{\boldpi,p^{\star}}\Big[\sum_{h=1}^H \bm{c}^{\star}\big(\vs_h,\va_h\big)\Big]\Big) \leq 0.
\end{align}
The reward and consumption regrets are therefore adapted to:
\begin{align*}
    &\textsc{ConvexRewReg}(k)\coloneqq f\Big(\Exp^{\boldpist, p^{\star}}\Big[\sum_{h=1}^{H}r^{\star}\big(\vs_h,\va_h\big)\Big]\Big) -f\Big(\frac{1}{k}\sum_{t=1}^{k}\Exp^{\boldpi_t, p^{\star}}\Big[\sum_{h=1}^{H}r^{\star}\big(\vs_h,\va_h\big)\Big]\Big), \\
    &\textsc{ConvexConsReg}(k)\coloneqq g \Big(\frac{1}{k}\sum_{t=1}^{k}\Exp^{\boldpi_t, p^{\star}}\Big[\sum_{h=1}^{H}\bm{c}^{\star}\big(\vs_h,\va_h \big)\Big]\Big).
\end{align*}

\xhdr{Our algorithm.} As in the basic setting, we wish to create a bonus-enhanced model and optimize over it. To model the transition
probabilites, we use empirical estimates $\smash{p^{(k)}}=\widehat{p}_k$ of Eq.~\eqref{eq:empirical-p} as before. However, since reward and consumption objectives are no longer monotone in the accumulated rewards and consumption respectively, it does not make sense to simply add or subtract $\smash{\widehat{\bonus}_k}\vphantom{r^{(k)}}$ (defined in Eq.~\ref{eq:bonus}) as we did before. Instead we compute the policy $\boldpi_k$ of episode $k$ together with the model by solving the following optimization problem which we call $\textsc{ConvexConPlanner}$:
\begin{align*}
\adjustlimits\max_{\boldpi}\max_{\!\!r^{(k)}\in\left[\widehat{r}_k\pm\widehat{\bonus}_k\right]}
  f\Big(\Exp^{\boldpi,p^{(k)}}\Big[\sum_{h=1}^H r^{(k)}\big(\vs_h,\va_h\big)\Big]\Big)
  \ \text{s.t.}
\hspace{-0.12in}
  \min_{\bm{c}^{(k)}\in \left[\widehat{\bm{c}}_k\pm \widehat{\bonus}_k\cdot \bm{1}\right]}
  g\Big(\Exp^{\boldpi,p^{(k)}}\Big[\sum_{h=1}^H \bm{c}^{(k)}\big(\vs_h,\va_h\big)\Big]\Big) \leq 0.
\end{align*}
The above problem is convex in the occupation measures,%
\footnote{Under mild assumptions, this program can be solved in polynomial time similar to its bandit analogue of Lemma 4.3 in \citep{AgrawalDevanurEC14}. We note that in the basic setting, it reduces to just a linear program.}
i.e., the probability $\rho(s,a,h)$ that the learner is at state-action-step $(s,a,h)$ --- c.f.~Appendix~\ref{app:subsec_convexConPlanner} for further discussion.
\begin{align*}
\adjustlimits\max_{\rho}\max_{r\in\left[\widehat{r}_k\pm\widehat{\bonus}_k\right]}
   f\Big(\sum_{s,a,h}\rho(s,a,h) r(s,a)\Big) \quad
&\text{s.t.}
\min_{\bm{c}\in \left[\widehat{\bm{c}}_k\pm \widehat{\bonus}_k\cdot \bm{1}\right]}
   g\Big(\sum_{s,a,h}\rho(s,a,h) \bm{c}(s,a)\Big) \leq 0
\\
\forall s',h:&\quad  \sum_{a} \rho(s',a,h+1)=\sum_{s, a }\rho(s,a,h)\widehat{p}_k(s'|s,a)\\
 \forall s,a,h:&\quad 0\leq \rho(s,a,h)\leq 1 \quad \text{and} \quad \sum_{s,a}\rho(s,a,h)=1.
\end{align*}

\xhdr{Guarantee for concave-convex setting.} To extend the guarantee of the basic setting to the concave-convex setting, we face an additional challenge: it is not immediately clear that the optimal policy $\boldpi^{\star}$ is feasible for the $\textsc{ConvexConPlanner}$ program because $\textsc{ConvexConPlanner}$ is defined with respect to the empirical transition probabilities $p^{(k)}$.%
\footnote{Note that in multi-armed bandit concave-convex setting \citep{AgrawalDevanurEC14}, proving feasibility of the best arm is straightforward as there are no transitions.}
\tledit{Moreover, when $H>1$, it is not straightforward to show that objective in the used model is always greater than the one in the true model as the  used model transitions  $p^{(k)}(s,a)$ can lead to different states than the ones encountered in the true model}.%
\footnote{\mdedit{Again, this is not an issue in multi-armed bandits.}}
We \tledit{deal with both of these issues by introducing} 
a novel application of \tledit{the} mean-value theorem to show that $\boldpi^{\star}$ is indeed a feasible solution of that program 
\tledit{and} create a similar regret decomposition to Proposition~\ref{prop:regret_decomposition} \wsedit{(see Proposition~\ref{lem:feasible_opt_convex} and more discussion in Appendix~\ref{app:feasible_mean_value})}\tledit{; this} 
allows us to plug in the results developed for the basic setting.  The full proof is provided in Appendix~\ref{app:convex}.

\begin{theorem}\label{thm:convex}
Let $L$ be the Lipschitz constant for $f$ and $g$ and let $\textsc{RewReg}$ and $\textsc{ConsReg}$ be the reward and consumption regrets for the basic setting (Theorem~\ref{thm:tabular}) \mdedit{with the failure probability $\delta$}. With probability $1-\delta$, our algorithm in the concave-convex setting has reward and consumption regret upper bounded by  $L\cdot\textsc{RewReg}$ and $Ld\cdot\textsc{ConsReg}$ respectively.
\end{theorem}

The linear dependence on $d$ in the consumption regret above comes from the fact that we assume $g$ is Lipschitz under $\ell_1$ norm. 

\section{Knapsack setting} \label{ssec:knapsacks}
\label{sec:knapsacks}

Our last technical section extends the algorithm and guarantee of the basic setting to \mdedit{scenarios} where the constraints are hard which is in accordance with most of the literature on \emph{bandits with knapsacks}. The goal here is to achieve aggregate reward regret that is sublinear in the \mdedit{time horizon} (in our case, the number of episodes $K$), while also respecting budget constraints for as small budgets as possible. We derive guarantees in terms of \emph{reward regret}, as defined previously, and then argue that \mdedit{our guarantee} extends to the seemingly stronger benchmark of the best dynamic policy.

\xhdr{Setting and objective.}
Each resource $i\in\resources$ has an aggregate budget $B_i$ that the learner should not exceed over $K$ episodes. \mdedit{Unlike the basic setting, where we track the consumption regret, here we view this as a hard constraint.} As in most works on bandits with knapsacks, the algorithm is allowed to use a ``null action'' for an episode, i.e., an action that \mdedit{yields a zero} reward and consumption when selected at the beginning of an episode. The learner wishes to maximize her aggregate reward while respecting these hard constraints. \mdedit{We reduce this problem to a specific variant of the basic problem~\eqref{eq:objective} with $\xi(i)=\frac{B_i}{K}$. We modify the solution to~\eqref{eq:objective} to take the null action if any constraint is violated and call the resulting benchmark $\pi^\star$.} Note that $\pi^\star$ satisfies constraints in expectation. At the end of this section, we explain how our algorithm also competes against a benchmark that is required to respect constraints deterministically \mdedit{(i.e., with probability one across all episodes)}.

\xhdr{Our algorithm.}
In the basic setting of Section~\ref{sec:basic_setting}, we showed a reward regret guarantee and a consumption regret guarantee, \mdedit{proving that the average constraint violation is $\bigO(1/\sqrt{K})$. Now we seek a stronger guarantee: the learned policy needs to satisfy budget constraints with high probability.} Our algorithm  optimizes a mathematical program $\textsc{KnapsackConPlanner}$ \eqref{eq:approx_program_knapsack} that strengthens the consumption \mdedit{constraints:}
\begin{align}\label{eq:approx_program_knapsack}
    \max_{\boldpi} \Exp^{\boldpi, p^{(k)}}\Big[\sum_{h=1}^H r^{(k)}\big(\vs_h,\va_h\big)\Big] \quad \text{s.t.} \quad \forall i\in\resources: \Exp^{\boldpi, p^{(k)}}\Big[\sum_{h=1}^H c^{(k)}\big(\vs_h,\va_h, i\big)\Big]\leq \frac{(1-\epsilon)B_i}{K}.
\end{align}
In the above, $p^{(k)}$, $r^{(k)}$, $\bm{c}^{(k)}$ are exactly as in the basic setting and $\epsilon>0$ is instantiated in the theorem below. Note that the program~\eqref{eq:approx_program_knapsack} is feasible thanks to the existence of the null action. The following mixture policy induces a feasible solution: with probability $1-\epsilon$, we play the optimal policy $\boldpi^\star$ for the entire episode; with probability $\epsilon$, we play the null action for the entire episode. Note that the above program can again be cast as a linear program in the occupancy measure space --- c.f.~Appendix~\ref{app:subsec_knapsackConPlanner} for further discussion.

\xhdr{Guarantee for knapsack setting.} The guarantee of the basic setting on this tighter mathematical program seamlessly transfers to a reward guarantee that does not violate the hard constraints.

\begin{theorem}\label{thm:knapsacks}
Assume that $\min_i B_i\leq KH$, \mdedit{i.e., constraints are non-vacuous}. Let $\textsc{AggReg}(\delta)$ be a bound on the aggregate (across episodes) reward or consumption regret  for the soft-constraint setting (Theorem~\ref{thm:tabular}) \mdedit{with the failure probability $\delta$}. Let $\smash[b]{\epsilon=\frac{\textsc{AggReg}(\delta)}{\min_i B_i}}$.  If $\min_i B_i>\textsc{AggReg}(\delta)$ then, with probability $1-\delta$, the reward regret in the hard-constraint setting
is at most $\smash[b]{\frac{2H\textsc{AggReg}(\delta)}{\min_i B_i}}$ and constraints are not violated.
\end{theorem}

The above theorem implies that the aggregate reward regret is sublinear in $K$ as long as $\min_i B_i \gg H\textsc{AggReg}(\delta)$. The analysis in the above main theorem (provided in Appendix~\ref{app:knapsacks}) is \emph{modular} in the sense that it leverages the \textsc{ConRL}'s performance to solve \eqref{eq:approx_program_knapsack} in a black-box manner. Smaller $\textsc{AggReg}(\delta)$ from the basic soft-constraint setting immediately translates to smaller reward regret and smaller budget regime (i.e., $\min_i B_i$ can be smaller). In particular, using the $\textsc{AggReg}(\delta)$ bound of Theorem~\ref{thm:tabular}, the reward regret is sublinear as long as $\min_i B_i=\Omega(\sqrt{K})$.

In contrast, previous work of \cite{Cheung19} can only deal with larger budget regime, i.e., $\min_i B_i=\Omega(K^{2/3})$. Although the guarantees are not directly comparable as the latter is for the single-episode setting, which requires further reachability assumptions, the budget we can handle is significantly smaller and in the next section we show that our algorithm has superior empirical performance in episodic settings even when such assumptions are granted.

\xhdr{Dynamic policy benchmark.} The common benchmark used in \mdedit{bandits with knapsacks} is not the best stationary policy $\pi^\star$ that respects constraints in expectation but rather the best \emph{dynamic} policy (i.e., a  policy that makes decisions based on the history) that never violates hard constraints \emph{deterministically}. In Appendix~\ref{app:knapsacks}, we show that the optimal dynamic policy (formally defined there) has reward less than  policy $\boldpi^{\star}$ (informally, this is because $\pi^\star$ respects constraints in expectation while the dynamic policy has to satisfy constraints deterministically) and therefore the guarantee of Theorem~\ref{thm:knapsacks} also applies against the optimal dynamic policy.

\section{Empirical comparison to other concave-convex approaches}
\label{sec:empirical}
\begin{figure*}
\centering
    \includegraphics[width=0.9\textwidth]{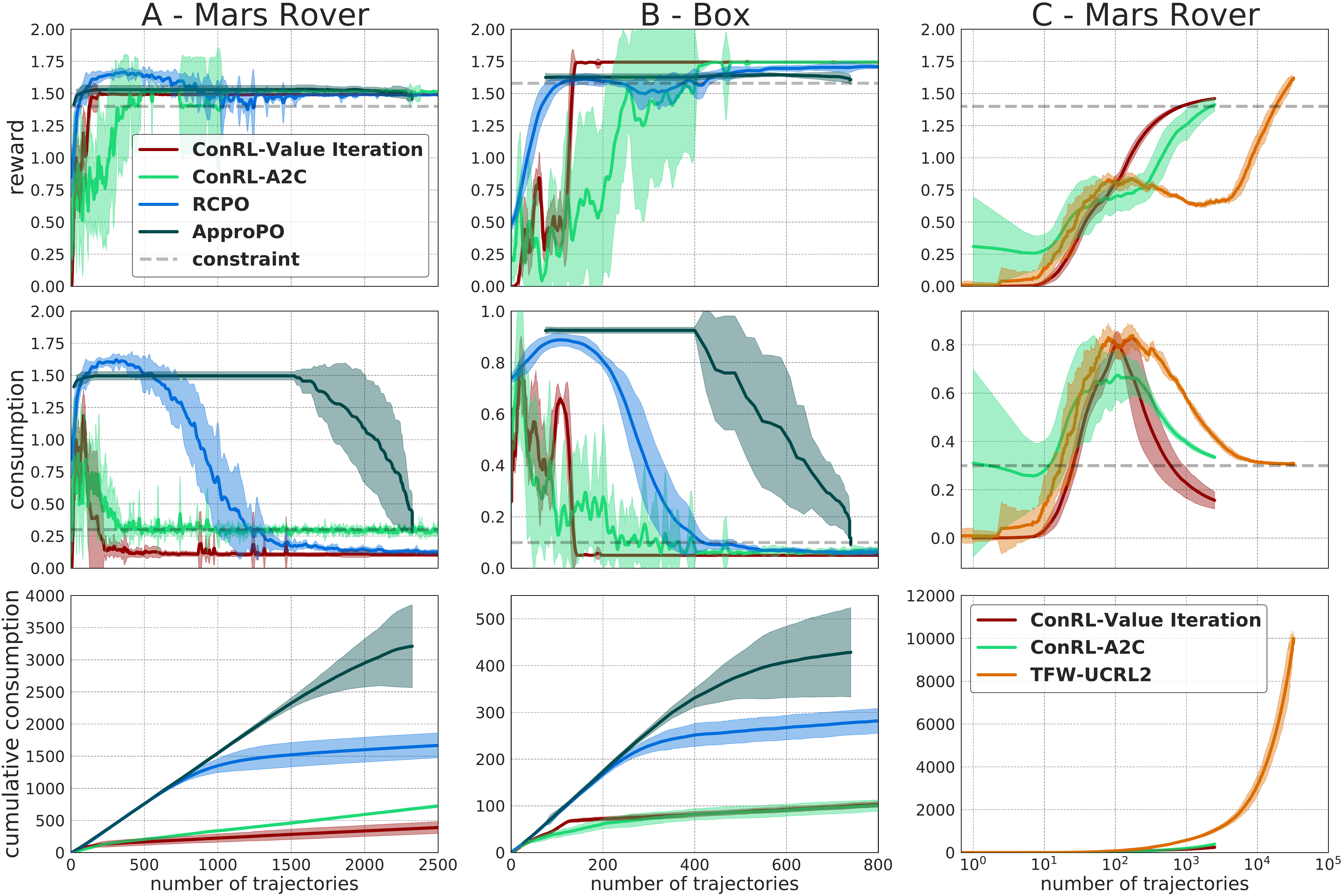}
    \caption{The performance of the algorithms as a function of the number of sample trajectories (trajectory $= 30$ samples);
    showing average and standard deviation over 10 runs. Dashed line in the second row is the upper bound on the consumption (for all algorithms), the dashed line in the first row is a lower bound on the reward (only required by $\textsc{ApproPO}$).}
    \label{fig:main}
    \vspace{-10pt}
\end{figure*}

In this section, we evaluate the performance of $\textsc{ConRL}$ against previous approaches.\footnote{Code is available at \url{https://github.com/miryoosefi/ConRL}} Although our \textsc{ConPlanner} (see Appendix~\ref{app:algorithm_details}) can be solved exactly using linear programming \citep{altman-constrainedMDP}, in our experiments, it suffices to use Lagrangian heuristic, denoted as \textsc{LagrConPlanner} (see Appendix \ref{app:lagr_conplanner}).
This Lagrangian heuristic only needs \mdedit{a planner for the \emph{unconstrained} RL task.  We consider two unconstrained RL algorithms as planners: value iteration and a model-based Advantage Actor-Critic (A2C) \citep{mnih2016asynchronous} (based on fictitious samples drawn from the model provided as an input). The resulting variants of \textsc{LagrConPlanner} are denoted \textsc{ConRL-Value Iteration} and \textsc{ConRL-A2C}}. We run our experiments on two grid-world environments \emph{Mars rover} \citep{tessler2018reward} and \emph{Box} \citep{leike2017ai}.%
\footnote{\tledit{We are not aware of any benchmarks for convex/knapsack constraints. For transparency, we compare against prior works handling concave-convex or knapsack settings on established benchmarks for the linear case.}}

\xhdr{Mars rover.} The agent must move from \mdedit{the initial} position to the goal without crashing into rocks.
If the agent reaches the goal or crashes into a rock it will stay in that cell for the remainder of the episode.
\mdedit{Reward is $1$ when the agent reaches the goal and $1/H$ afterwards.
Consumption is $1$ when the agent crashes into a rock and $1/H$ afterwards.}
The episode horizon $H$ is $30$ and the agent's action is perturbed with probability $0.1$ to a random action.

\xhdr{Box.} The agent must move \mdedit{a box from the initial position} to the goal while avoiding
\mdedit{corners} (cells adjacent to at least two walls).
If the agent reaches the goal it stays in that cell for the remainder of the episode.
Reward is $1$ when agent reaches the goal for the first time and $1/H$ afterwards; consumption is $1/H$ \mdedit{whenever the} box is in a corner. Horizon $H$ is $30$ and the agent's action is perturbed with probability $0.1$ to a random action.

We compare $\textsc{ConRL}$ to previous constrained approaches (derived for either episodic or single-episode settings) in Figure~\ref{fig:main}. \mdedit{We keep track of three metrics: episode-level reward and consumption (the first two rows) and cumulative consumption (the third row). Episode-level metrics are based on the most recent episode in the first two columns, i.e., we plot $\mathbb{E}^{\pi_k}[\sum_{h=1}^H r^{\star}_h]$ and $\mathbb{E}^{\pi_k}[\sum_{h=1}^H c^{\star}_h]$. In the third column, we plot the average across episodes so far, i.e., $\frac{1}{k}\sum_{t=1}^k\mathbb{E}^{\pi_t}[\sum_{h=1}^H r^{\star}_h]$ and $\frac{1}{k}\sum_{t=1}^k\mathbb{E}^{\pi_t}[\sum_{h=1}^H c^{\star}_h]$, and we use the log scale for the $x$-axis. The cumulative consumption is $\sum_{t=1}^k \sum_{h=1}^H c_{t,h}$ in all columns. See Appendix~\ref{app:experiment} for further details about experiments.}

\xhdr{Episodic setting.} We first compare our algorithms to two episodic RL approaches: $\textsc{ApproPO}$ \citep{MiryoosefiBrDaDuSc19} and $\textsc{RCPO}$ \citep{tessler2018reward}. We note that none of the previous approaches in this setting address sample-efficient exploration. In addition, most of them are limited to linear constraints, with the exception of $\textsc{ApproPO}$ \citep{MiryoosefiBrDaDuSc19}, which can handle general convex constraints.%
\footnote{In addition to that, trust region methods like CPO \citep{achiam2017constrained} address a more restrictive setting and require constraint satisfaction at each iteration; for this reason, they are not included in the experiments.}
\mdedit{Both $\textsc{ApproPO}$ and $\textsc{RCPO}$ (used as a baseline by \citealp{MiryoosefiBrDaDuSc19}) maintain and update a weight vector $\bm{\lambda}$, used to derive reward for an unconstrained RL algorithm, which we instantiate as A2C. $\textsc{ApproPO}$ focuses on the feasibility problem, so it requires to specify a lower bound on the reward, which we set to $0.3$ for Mars rover and $0.1$ for Box. In the first two columns of Figure \ref{fig:main} we see that both versions of $\textsc{ConRL}$ are able to solve the constrained RL task with a much smaller number of trajectories (see top two rows), and their overall consumption levels are substantially lower (the final row) than those of the previous approaches.}

\xhdr{Single-episode setting.}
Closest to our work is \textsc{TFW-UCRL2} \citep{Cheung19}, which is based on UCRL \citep{jaksch2010near}. However, that approach focuses on the single-episode setting and requires a strong reachability assumption. By connecting terminal states of our MDP to the intial state, we reduce our episodic setting to single-episode setting in which we can compare $\textsc{ConRL}$ against \textsc{TFW-UCRL2}. \mdedit{Results for Mars rover are depicted in last column of Figure~\ref{fig:main}.%
\footnote{Due to a larger state space, it was computationally infeasible to run  \textsc{TFW-UCRL2} in the Box environment.}
Again, both versions of \textsc{ConRL} find the solution with a much smaller number of trajectories (note the log scale on the $x$-axis) and their overall consumption levels are much lower than those of \textsc{TFW-UCRL2}}.
This suggests that \textsc{TFW-UCRL2} might be impractical in (at least some) episodic settings. 

 \section{Conclusions}
 \label{sec:conclusions}
 
\wsedit{In this paper we study two types of constraints in the framework of constrained tabular episodic reinforcement learning: concave rewards and convex constraints, and knapsacks constraints. Our algorithms achieve near-optimal regret in both settings, and experimentally we show that our approach outperforms prior works on constrained reinforcement learning.}

\wsedit{Regarding future work, it would be interesting to extend our framework to continuous state and action spaces. Potential directions include extensions to Lipschitz MDPs \citep{song2019efficient} and MDPs with linear parameterization \citep{jin2019provably} where optimism-based exploration algorithms exist under the classic reinforcement learning setting without constraints.}



\section*{Broader Impact}
Our work focuses on the theoretical foundations of reinforcement learning by addressing the important challenge of constrained optimization in reinforcement learning. We strongly believe that understanding the theoretical underpinnings of the main machine learning paradigms is essential and can guide principled and effective deployment of such methods.

Beyond its theoretical contribution, our work may help the design of reinforcement learning algorithms that go beyond classical digital applications of RL (board games and video games) and extend to settings with complex and often competing objectives. We believe that constraints constitute a fundamental limitation in extending RL beyond the digital world, as they exist in a wide variety of sequential decision-making applications (robotics, medical treatment, education, advertising). Our work provides a paradigm to design algorithms with efficient exploration despite the presence of constraints.\looseness=-1

That said, one needs to ensure that an algorithm offers acceptable quality in applications. Any exploration method that does not rely on off-policy samples will inevitably violate constraints \emph{sometimes} in order to learn. In some applications, this is totally acceptable: a car staying out of fuel in rare circumstances is not detrimental, an advertiser exhausting their budget some month is even less significant, a student dissatisfaction in an online test is unpleasant but probably acceptable. On the other hand, if the constraint violation involves critical issues like drug recommendation for severe diseases or decisions by self-driving cars that can cause physical harm to passengers then the algorithm needs to be carefully reviewed. It may be necessary to ``prime'' the algorithm with some data collected in advance (however costly it may be). One may need to make a judgement call on whether the ethical or societal standards are consistent with deploying an algorithm in a particular setting.\looseness=-1

To summarize, our work is theoretical in nature and makes significant progress on a problem at the heart of RL. It has the potential to guide deployment of constrained RL methods in many important applications and tackle a fundamental bottleneck in deploying RL beyond the digital world. However, 
an application needs to be carefully reviewed before deployment. 

\begin{ack}
The authors would like to thank Rob Schapire for useful discussions that helped in the initial stages of this work \added{and Yufeng Zhang whose careful reading of our proofs uncovered an error leading to an incorrect dependence on~$H$ in Theorem~\ref{thm:tabular}}. Part of the work was done when WS was at Microsoft Research NYC.\looseness=-1
\end{ack}



\bibliographystyle{apalike}
\bibliography{bibliog,extras}

\newpage
\appendix



\section*{Structure of the supplementary material.} The supplementary material consists of six sections:
\begin{itemize}
    \item Appendix~\ref{app:algorithm_details} provides the formal description of the algorithm and the instantiations of \textsc{ConPlanner} as well as how it can be expressed as a (linear/convex) mathematical program.
    \item Appendix~\ref{app:basic_setting} provides the proofs for the results of the basic setting presented in Section~\ref{sec:basic_setting}.
    \item Appendix~\ref{app:convex} provides the proofs and additional discussion for the results of the concave-convex setting presented in Section~\ref{sec:concave-convex}.
    \item Appendix~\ref{app:knapsacks} provides the proofs and additional discussion for the results of knapsack setting presented in Section~\ref{sec:knapsacks}.
    \item Appendix~\ref{app:experiment} provides further details regarding the experiments presented in Section~\ref{sec:empirical}.
    \item Appendix~\ref{app:auxil} provides auxiliary concentration lemmas useful for the derivation of our results.
\end{itemize}.

\section{Algorithm: Formal description and design choices}
\label{app:algorithm_details}
Our main algorithm, denoted by \textsc{ConRL}, is presented at Algorithm~\ref{alg:ConRL}. We instantiate \textsc{ConRL} for our different settings (i.e. basic setting, concave-convex, and knapsack) by using the appropriate $\textsc{ConPlanner}$ that we discuss in the remainder of this section.
\begin{algorithm}[H]
 \begin{algorithmic}[1]
 \FOR{Episode $k$ from $1$ to $K$}
     \STATE \textbf{Compute empirical estimates:} \\\begin{center}
         Compute $N_k$, $\widehat{p}_k$, $\widehat{r}_k$, and $\widehat{\bm{c}}_k$ based on equations (\ref{eq:empirical-N}-\ref{eq:empirical-c})
     \end{center}
     \STATE \textbf{Compute bonus:}\begin{center}
          Compute $\widehat{b}_k$ as equation (\ref{eq:bonus})
     \end{center}
     \STATE \textbf{Call constrained planner}: \begin{center}
         $\boldpi_k \leftarrow \textsc{ConPlanner}(\widehat{p}_k,\widehat{r}_k,\widehat{\bm{c}}_k, \widehat{b}_k)$
     \end{center}
     \STATE \textbf{Execute policy}: initial state  $\vs_{k,1} = s_0$\;
     \FOR{Stage $h$ from $1$ to $H$}
     \STATE Select $\va_{k,h} \sim \boldpi_k\Big(\vs_{k,h}\Big)$\;
     \STATE Observe	reward $r_{k,h}$, consumptions $\forall i\in \resources: c_{k,h,i}$, and new state  $\vs_{k,h+1}$\;
     \ENDFOR
 \ENDFOR
\caption{$\textsc{ConRL}$}
\label{alg:ConRL}
 \end{algorithmic}
\end{algorithm}

\subsection{Basic setting - \textsc{BasicConPlanner}}\label{app:subsec_basicConPlanner} We define the bonus-enhanced $\cmdp$, i.e. $\model^{(k)}=\big(p^{(k)},r^{(k)},\bm{c}^{(k)}\big)$, as
\begin{align*}
      &p^{(k)}(s' | s,a) = \widehat{p}_k(s' | s,a ) \quad \forall s,a,s'\\
      &r^{(k)}(s,a) = \widehat{r}_k(s,a)+\widehat{b}_k(s,a)\quad  \forall s,a \\
      &c^{(k)}(s,a,i) = \widehat{c}_k(s,a,i)-\widehat{b}_k(s,a)\quad \forall s,a , i \in \resources\\
\end{align*}
then we solve the following optimization problem 
\begin{align*}
    &\max_{\boldpi} \Exp^{\boldpi, p^{(k)}}\Big[\sum_{h=1}^H r^{(k)}\big(\vs_h,\va_h\big)\Big] \qquad \text{s.t.} \qquad &\forall i\in\resources: \Exp^{\boldpi, p^{(k)}}\Big[\sum_{h=1}^H c^{(k)}\big(\vs_h,\va_h, i\big)\Big]\leq \xi(i).
\end{align*} 
This optimization problem can be solved exactly since it is equivalent to the following linear program on occupation measures \citep{rosenberg2019online,altman-constrainedMDP}. Decision variables are $\rho(s,a,h)$, i.e. probability of agent being at state action pair $(s,a)$ at time step $h$.
\begin{equation}
\label{eq:LP_basicSetting}
\begin{aligned}
    \max_{\rho} \sum_{s,a,h} \rho(s,a,h)r^{(k)}(s,a) 
    &\quad \text{ s.t. } \sum_{s,a,h} \rho(s,a,h)c^{(k)}(s,a,i) \leq \xi(i) \quad \forall i \in \resources\\
    &\forall s',h \quad \sum_{a} \rho(s',a,h+1)=\sum_{s,a} \rho(s,a,h)p^{(k)}(s' | s,a) \\
    &\forall s,a,h \quad  0 \leq \rho(s,a,h) \leq 1 \quad   \quad \sum_{s,a} \rho(s,a,h)=1 
\end{aligned}
\end{equation}
\subsection{Concave-convex setting - \textsc{ConvexConPlanner}}\label{app:subsec_convexConPlanner} In this setting, unlike basic setting, objective and constraints are not linear. Therefore, due to lack of monotonicity, we cannot explicitly define the bonus-enhanced $\cmdp$ $\model^{(k)}=\big(p^{(k)},r^{(k)},\bm{c}^{(k)}\big)$. The bonus-enhanced $\cmdp$ is implicit in the following program that we solve (see \autoref{sec:concave-convex})
\begin{align*}
\max_{\boldpi}\hspace{-0.09in}\max_{r^{(k)}\in\big[\widehat{r}_k\pm\widehat{\bonus}_k\big]}f\Big(\Exp^{\boldpi,p^{{(k)}}}\Big[\sum_{h=1}^H r^{(k)}\big(\vs_h,\va_h\big)\Big]\Big) \ \text{s.t.}\hspace{-0.06in} \min_{\bm{c}^{(k)}\in \big[\widehat{\bm{c}}_k\pm \widehat{\bonus}_k\cdot \bm{1}\big]}g\Big(\Exp^{\boldpi,p^{(k)}}\Big[\sum_{h=1}^H \bm{c}^{(k)}\big(\vs_h,\va_h\big)\Big]\Big) \leq 0.
\end{align*}
Similar to before, expressing this program based on occupation measures provides a convex program.
\begin{equation}
\label{eq:CP_concaveConvexSetting}
\begin{aligned}
    \max_{\rho} \hspace{-0.09in} \max_{r\in \big[\widehat{r}_k\pm\widehat{\bonus}_k\big]} f\Big(\sum_{s,a,h}\rho(s,a,h)r(s,a)\Big)\quad &\text{s.t.}\  \min_{\bm{c}\in \big[{\widehat{\bm{c}}}_k\pm\widehat{\bonus}_k\cdot\boldsymbol{1}\big]} g\Big(\sum_{s,a,h}\rho(s,a,h)\bm{c}(s,a)\Big)\leq 0 \\ &{\forall s',h:\quad  \sum_{a} \rho(s',a,h+1)=\sum_{s, a }\rho(s,a,h)\widehat{p}_k(s'|s,a)}\\
 &\forall s,a,h:\quad 0\leq \rho(s,a,h)\leq 1 \quad \text{and} \quad \sum_{s,a}\rho(s,a,h)=1
\end{aligned} 
\end{equation}
The notations ${r\in \big[\widehat{r}_k\pm\widehat{\bonus}_k\big]}$ and ${\bm{c}\in \big[{\widehat{\bm{c}}}_k\pm\widehat{\bonus}_k\cdot\boldsymbol{1}\big]}$ are defined as
$$
    {r\in \big[\widehat{r}_k\pm\widehat{\bonus}_k\big]} \iff \forall s,a: \quad  r(s,a) \in [\widehat{r}_k(s,a)-\widehat{\bonus}_k(s,a),\widehat{r}_k(s,a)+\widehat{\bonus}_k(s,a)]
$$
$$
    {\bm{c}\in \big[{\widehat{\bm{c}}}_k\pm\widehat{\bonus}_k\cdot\boldsymbol{1}\big]} \iff  \forall i \in \resources, s, a: \quad c(s,a,i) \in [\widehat{c}_k(s,a,i)-\widehat{\bonus}_k(s,a),\widehat{c}_k(s,a,i)+\widehat{\bonus}_k(s,a)]
$$
Note that if $f$ and $g$ are linear, we end up with a linear program similar to (\ref{eq:LP_basicSetting})
\subsection{Knapsack setting - \textsc{KnapsackConPlanner}}\label{app:subsec_knapsackConPlanner}
We define the bonus-enhanced cMDP, i.e. $\model^{(k)}=\big(p^{(k)},r^{(k)},\bm{c}^{(k)}\big)$ similar to basic setting (\ref{app:subsec_basicConPlanner}). We also solve a similar optimization problem with tighter constraints:
\begin{align*}
    \max_{\boldpi} \Exp^{\boldpi, p^{(k)}}\Big[\sum_{h=1}^H r^{(k)}\big(\vs_h,\va_h\big)\Big] \quad \text{s.t.} \quad \forall i\in\resources: \Exp^{\boldpi, p^{(k)}}\Big[\sum_{h=1}^H c^{(k)}\big(\vs_h,\va_h, i\big)\Big]\leq \frac{(1-\epsilon)B_i}{K}.
\end{align*} This optimization problem can again be solved using the following linear program on occupation measures. Decision variables are $\rho(s,a,h)$, i.e. probability of agent being at state action pair $(s,a)$ at step $h$.
\begin{equation}
\label{eq:LP_knapsackSetting}
\begin{aligned}
    \max_{\rho} \sum_{s,a,h} \rho(s,a,h)r^{(k)}(s,a) 
    &\quad \text{ s.t. } \sum_{s,a,h} \rho(s,a,h)c^{(k)}(s,a,i) \leq \frac{(1-\epsilon)B_i}{K} \quad \forall i \in \resources\\
    &\forall s',h \quad \sum_{a} \rho(s',a,h+1)=\sum_{s,a} \rho(s,a,h)p^{(k)}(s' | s,a) \\
    &\forall s,a,h \quad  0 \leq \rho(s,a,h) \leq 1 \quad   \quad \sum_{s,a} \rho(s,a,h)=1 
\end{aligned}
\end{equation}

\section{Analysis: Basic setting (Section~\ref{sec:basic_setting})}
\label{app:basic_setting}

In this section, we prove the main guarantee for the basic setting. 

\subsection{Validity of bonus (Lemma~\ref{lem:valid_bonus})}\label{app:valid_bonus}
\wsedit{We first prove that $\widehat{\bonus}_k(s,a)= \min\left\{2 H, H\sqrt{\frac{2\ln\big(8SAH(d+1)k^2/\delta)}{N_k(s,a)}}\right\}$ } of Eq. \eqref{eq:bonus} is valid as in the Definition~\ref{defn:valid_bonus}.
\begin{proof}[Proof of Lemma~\ref{lem:valid_bonus}]
We focus on a single state-action pair $s,a$, stage $h$, and objective $m$. Since the support of $m$ is in $[0,1]$ and the one of the value is in $[0,H-1]$, by Hoeffding inequality  (see Lemma~\ref{lem:azuma_hoeffding_anytime}), it holds that, for all $k$, since $(s,a)$-pair is visited $N_k(s,a)$ times prior to episode $k$, with probability at least $1-\delta'$:
\begin{align*}\Big|\Big(\widehat{m}_k(s,a)-m^{\star}(s,a)\Big)+\sum_{s'\in\states}\Big(\widehat{p}_k(s'|s,a)-p^{\star}(s'|s,a)\Big)V
\Big|\leq H\sqrt{\frac{2\ln(2/\delta')}{N_k(s,a)}}.
\end{align*}
Also note that $\widehat{m}_k(s,a) \in [0,1], m^{\star}(s,a)\in [0,1]$, and $\|V\|_{\infty} \leq H$, the LHS of the above inequality must be less than $1 + H \leq 2H$.

As a result, the bonus $\widehat{\bonus}_k(s,a,\delta)$ satisfies this inequality for a particular state-action-step-objective with failure probability at most $\delta'=\frac{\delta}{4SAH(d+1)k^2}$ and is therefore valid (satisfying it for all states-actions-steps-objectives) with failure probability $\frac{\delta}{4k^2}$.
Union bounding across episodes, the probability of $\widehat{\bonus}_k(s,a,\delta)$ not being valid for some $k$ is at most $\sum_{k=1}^K \frac{\delta}{4k^2}\leq \delta$.
\end{proof}

\subsection{Valid bonus implies optimism}
The main reason to optimize a bonus-enhanced model with valid bonuses is because the latter render the model \emph{optimistic}, i.e., its estimated reward is an overestimate of the true reward. Similarly, in constrained settings, its estimated resource consumptions are underestimates of the true resource consumptions. This is formalized in the following definition.
\begin{definition}\label{defn:optimistic_model}
A $\cmdp$ $\model=(p,r,\bm{c})$ is \emph{optimistic} if its estimated reward (resp. consumption) value function for policy $\boldpist$ upper (resp. lower) bounds its corresponding value function under the ground truth:  \begin{align*}&\Exp\Big[V_r^{\boldpist,p}(\vs_1,1)\Big]\geq \Exp\Big[V_{r^{\star}}^{\boldpist,p^{\star}}(\vs_1,1)\Big]\quad \text{and}\quad \Exp\Big[V_{c_i}^{\boldpist,p}(\vs_1,1)\Big]\leq \Exp\Big[V_{c_i^{\star}}^{\boldpist,p^{\star}}(\vs_1,1)\Big] \forall i\in\resources.\end{align*}
\end{definition}
An important block of the analysis for the basic setting is to show that, when using a bonus-enhanced model with valid bonuses, the resulting $\cmdp$ is optimistic.

\begin{lemma}\label{lem:valid_implies_optimism}
If the bonus $\widehat{\bonus}_k(s,a)$ of Eq.~\eqref{eq:bonus} in episode $k$ is valid (Definition~\ref{defn:valid_bonus}) for the corresponding $\cmdp$ $\model^{(k)}=\big(p^{(k)},r^{(k)},\bm{c}^{(k)}\big)$ then $\model^{(k)}$ is optimistic.
\end{lemma}
\begin{proof}
We first prove the optimism of the model for the reward objective. More concretely, we show by induction that for any state $s$, action $a$, and stage $h$, $Q^{\boldpist,p^{(k)}}_{r^{(k)}}(s,a,h)\geq Q^{\boldpist,p^{\star}}_{r^{\star}}(s,a,h)$; taking expectation on the state-action pair of the first state, the claim then follows. 

{Since the setting ends at episode $H$, $Q^{\boldpist,p^{(k)}}_{r^{(k)}}(s,a,H+1)= Q^{\boldpist,p^{\star}}_{r^{\star}}(s,a,H+1)=0$. 

We assume that the inductive hypothesis $Q^{\boldpist,p^{(k)}}_{r^{(k)}}(s,a,h+1)\geq Q^{\boldpist,p^{\star}}_{r^{\star}}(s,a,h+1)$ (and thus also $V^{\boldpist,p^{(k)}}_{r^{(k)}}(s,h+1)\geq V^{\boldpist,p^{\star}}_{r^{\star}}(s,h+1)$) holds, and proceed with the inductive step. The $Q$-functions in question are:}
\begin{align*}
Q^{\boldpist,p^{(k)}}_{r^{(k)}}(s,a,h)&=r^{(k)}(s,a)+\sum_{s'\in\states}p^{(k)}(s'|s,a)V^{\boldpist,p^{(k)}}_{r^{(k)}}(s',h+1)\\&\geq  r^{(k)}(s,a)+\sum_{s'\in\states}p^{(k)}(s'|s,a)V^{\boldpist,p^{\star}}_{r^{\star}}(s',h+1)\\
Q^{\boldpist,p^{\star}}_{r^{\star}}(s,a,h)&=r^{\star}(s,a)+\sum_{s'\in\states}p^{\star}(s'|s,a)V^{\boldpist,p^{\star}}_{r^{\star}}(s',h+1)
\end{align*}
Subtracting, we have:
\begin{align*}Q^{\boldpist,p^{(k)}}_{r^{(k)}}(s,a,h)-Q^{\boldpist,p^{\star}}_{r^{\star}}(s,a,h)&\geq\Big(\widehat{r}_k(s,a)+\widehat{\bonus}_k(s,a)-r^{\star}(s,a)\Big)\\&+\sum_{s'\in\states}\Big(\widehat{p}_k(s'|s,a)-p^{\star}(s'|s,a)\Big)V^{\boldpist,p^{\star}}_{r^{\star}}(s',h+1)\geq 0,
\end{align*}
where the last inequality holds since the bonuses are valid. 

The optimism of the model with respect to the consumption objectives follows the same steps altering the direction of the inequalities and setting the estimate as empirical mean minus the bonus. 
\end{proof}

We emphasize that our bonus in Eq~\eqref{eq:bonus} does not scale polynomially with respect to $|\mathcal{S}|$; despite that, as indicated by the above lemma, it suffices to prove optimism.

\subsection{Simulation lemma}\label{app:simulation_lemma}
To prove the Bellman-error regret decomposition, an essential piece is the so called \emph{simulation lemma} \citep{Kearns2002} which we adapt to constrained settings below:
\begin{lemma}[Simulation lemma]\label{lem:simulation}
For any policy $\boldpi$, any $\cmdp$ $\model=(p,r,\bm{c})$, and any objective $m\in \{r\}\cup \{c_i\}_{i\in\resources}$ with corresponding true objective $m^{\star}\in \{r^{\star}\}\cup \{c_i^{\star}\}_{i\in\resources},$, it holds that:
\begin{align}
    &\Exp^{\boldpi}\Big[V_m^{\boldpi,p}(\vs_1,1)\Big]-\Exp^{\boldpi}\Big[V_{m^{\star}}^{\boldpi,p^{\star}}(\vs_1,1)\Big]
    = \Epi\Big[\sum_{h=1}^H \bellman_m^{\boldpi,p}(\vs_h,\va_h,h)\Big].\label{eq:bound_values_bellman}
\end{align}
\end{lemma}
\begin{proof}
For all of $m\in\{r\}\cup \{c_i\}_{i\in\resources}$, rearranging the definitions of Bellman errors, we obtain: 
\begin{align*}Q_m^{\boldpi,p}(s,a,h)&=\Big(\bellman_m^{\boldpi,p}(s,a,h)+m^{\star}(s,a)\Big)+\sum_{s'\in\states}p^{\star}(s'|s,a) V_m^{\boldpi,p}(s',h+1)\\
Q_{m^{\star}}^{\boldpi,p^{\star}}(s,a,h)&=\Big(\bellman_{m^{\star}}^{\boldpi,p^{\star}}(s,a,h)+m^{\star}(s,a)\Big)+\sum_{s'\in\states}p^{\star}(s'|s,a) V_{m^*}^{\boldpi,p^*}(s',h+1)
\end{align*}
By definition of the Bellman error, the Bellman error with respect to the true model is equal to $0$. As a result, subtracting the two above equations, we obtain:
\begin{align*}
    Q_{m}^{\boldpi,p}(s,a,h)-Q_{m^{\star}}^{\boldpi,p^{\star}}(s,a,h)=\bellman_m^{\boldpi,p}(s,a,h) + \sum_{s'\in\states} p^{\star}(s'|s,a)\Big(V_m^{\boldpi,p}(s',h+1) -   V_{m^{\star}}^{\boldpi,p^{\star}}(s',h+1)\Big).
\end{align*}
Taking expectation over policy $\boldpi$ to select $a$, the initial state $s_1$, and setting $h=1$, we obtain:
\begin{align*}
    \Exp_{\vs_1}\Big[V_m^{\boldpi,p}\big(\vs(1),1\big)-V_{m^{\star}}^{\boldpi,p^{\star}}\big(\vs_1,1\big)\Big]&=\Exp^{\boldpi}\Big[\bellman_m^{\boldpi,p}\big(\vs_1,\va_1,1\big)\Big]+\Exp^{\boldpi}\Big[V_m^{\boldpi,p}\big(\vs_2,2\big) -   V_{m^{\star}}^{\boldpi,p^*}\big(\vs_2,2\big)\Big].
\end{align*}
Recursively bounding the second term of the RHS as above concludes the lemma.
\end{proof}

\subsection{Bellman-error regret decomposition (Proposition~\ref{prop:regret_decomposition})}

\begin{proof}[Proof of Proposition~\ref{prop:regret_decomposition}]
The consumption requirement \eqref{eq:consumption_decomposition} for resource $i$ follows by applying the simulation lemma (Lemma~\ref{lem:simulation}) on $\cmdp$ $\model^{(k)}$ and objective $m=c_i^{(k)}$ (with corresponding true objective $m^{\star}=c_{i}^{\star}$) and using that $\boldpi_k$ is feasible for $\textsc{ConPlanner}(p^{(k)},r^{(k)},\bm{c}^{(k)})$:
\begin{align*}
  \Exp^{\boldpi_k, p^{\star}}\Big[\sum_{h=1}^Hc^{\star}(\vs_h,\va_h,i)\Big]&=  
    \Exp^{\boldpi_k}\Big[V_{c_i^{\star}}^{\boldpi,p^{\star}}(\vs_1,1)\Big]=\Exp\Big[V_{c_i}^{\boldpi_k,p}(\vs_1,1)\Big]- \Exp^{\boldpi_k}\Big[\sum_{h=1}^H \bellman_{c_i^{(k)}}^{\boldpi_k,p^{(k)}}\big(\vs_h,\va_h,h)\Big]\\
    &\leq \xi(i)+\Exp^{\boldpi_k}\Big[\sum_{h=1}^H \Big| \bellman_{c_i^{(k)}}^{\boldpi_k,p^{(k)}}\big(\vs_h,\va_h,h\big)\Big|\Big]
\end{align*}
Regarding the reward requirement \eqref{eq:reward_decomposition}, what we wish to bound is:
\begin{align*}\Exp^{\boldpist, p^{\star}}\Big[\sum_{h=1}^Hr^{\star}(\vs_h,\va_h)\Big]-\Exp^{\boldpi_k, p^{\star}}\Big[\sum_{h=1}^Hr^{\star}(\vs_h,\va_h)\Big]= \Exp\Big[V_{r^{\star}}^{\boldpist,p^{\star}}(\vs_1,1)\Big]-\Exp\Big[V_{r^{\star}}^{\boldpi_k,p^{\star}}(\vs_1,1)\Big]
\end{align*}
the validity of the bonus implies that the model $\model^{(k)}$ is optimistic (Lemma~\ref{lem:valid_implies_optimism}), i.e., we have that $\Exp\Big[V_{r^{\star}}^{\boldpist,p^{\star}}(\vs_1,1)\Big]\leq \Exp\Big[V_{r^{(k)}}^{\boldpist,p^{(k)}}(\vs_1,1)\Big]$.
If $\boldpist$ is feasible for $\textsc{ConPlanner}(p^{(k)},r^{(k)},\bm{c}^{(k)})$ then, since $\boldpi_k$ is the maximizer for this program:
\begin{align}\Exp\Big[V_{r^{(k)}}^{\boldpist,p^{(k)}}(\vs_1,1)\Big]-\Exp\Big[V_{r^{\star}}^{\boldpi_k,p^{\star}}(\vs_1,1)\Big]&\leq \Exp\Big[V_{r^{(k)}}^{\boldpi_k,p^{(k)}}(\vs_1,1)\Big]-\Exp\Big[V_{r^{\star}}^{\boldpi_k,p^{\star}}(\vs_1,1)\Big]\label{eq:need_exactness}
\\&=  \Exp^{\boldpi_k}\Big[\sum_{h=1}^H \bellman_{r^{(k)}}^{\boldpi_k,p^{(k)}}\big(\vs_h,\va_h,h\big)\Big]\nonumber
\end{align}
where the last equality holds by applying the simulation lemma with $m=r$. Hence, this proves \eqref{eq:reward_decomposition}.

What is left to show is that $\boldpist$ is indeed feasible for $\textsc{ConPlanner}(p^{(k)},r^{(k)},\bm{c}^{(k)})$. Since $\model^{(k)}$ is optimistic and $\boldpist$ is feasible for the ground truth $\truemodel$, for all resources $i\in\resources$:
\begin{align*}\Exp\Big[V_{c_i^{(k)}}^{\boldpist,p^{(k)}}(\vs_1,1)\Big]\leq \Exp\Big[V_{c_i^{\star}}^{\boldpist,p^{\star}}(\vs_1,1)\Big]\leq \xi(i).
\end{align*}
This completes the proof of the proposition.
\end{proof}

\subsection{Bounding the Bellman error}
\label{app:covering_bellman_error}
We now provide an upper bound on the Bellman error which arises in the RHS of the regret decomposition (Proposition~\ref{prop:regret_decomposition}).
\begin{lemma}\label{lem:bellman_error_bound}
Let $\epsilon>0$. If the bonus $\widehat{\bonus}_k$ is valid for all episodes $k$ simultaneously then, with probability at least $1-\delta$:
for all objectives $m^{(k)}\in\{r^{(k)}\}\cup \{c_{i}^{(k)}\}_{i\in\resources}$, transitions $p=p^{(k)}$, and stages $h$, the Bellman error at episode $k$ is upper bounded by:
\begin{align*}
    &\Big|\bellman_{m^{(k)}}^{\boldpi_k ,p^{(k)}}(\vs,\va,h)\Big|
    \leq \wsedit{4H^2}\sqrt{\frac{2S\ln\big(16SAH^2(d+1)k^2/(\epsilon \delta))}{N_k(s,a)}}+\epsilon S.
\end{align*}
\end{lemma}
\begin{proof}[Proof of Lemma~\ref{lem:bellman_error_bound}]
\wsedit{Let $\Psi$ be an $\epsilon$-net in $[-2H^2,2H^2]^S$}. For a fixed value $\bar{V}\in \Psi$, similar to Lemma~\ref{lem:valid_bonus}, with probability $1-\delta'$, simultaneously for all states $s\in\states$, actions $a\in\actions$, steps $h\in[H]$, episodes $k\in[K]$, and objectives $m^{(k)}\in\{r^{(k)}\}\cup \{c_{i}^{(k)}\}_{i\in\resources}$, it holds that: 
\begin{align*}
    \Big|m^{(k)}(s,a)-m^{\star}(s,a)&+\sum_{s'\in\states}\Big(p(s'|s,a)-p^{\star}(s'|s,a)\Big)\bar{V}(s')\Big|\\
    &\leq \widehat{\bonus}_k(s,a)+\wsedit{2H^2}\sqrt{\frac{2\ln\big(8SAH(d+1)k^2/\delta')}{N_k(s,a)}}
\end{align*}
Since $\Psi$ is an $\epsilon$-net for $\bar{V}$, there are $(2H^2/\epsilon)^S$ potential values. In order to have the above hold simultaneously for all these values with probability $1-\delta$, we need to set $\delta'= \frac{\delta}{(2H^2/\epsilon)^S}$.

\wsedit{Since the value $\big(p^{(k)}(s'|s,a)-p^{\star}(s'|s,a)\big)V_{m^{(k)}}^{\boldpi_k,p}(s',h+1)$ is in $[-2H^2,2H^2]$} for all $s'$, it holds that there exists a value $V$ in the $\epsilon$-net with distance at most $\epsilon S$. As a result, since $\widehat{\bonus}_k(s,a)$ is valid for $k$:
\begin{align*}
    \Big|\bellman_{m^{(k)}}^{\boldpi_k,p^{(k)}}(s,a,h)\Big|&\leq \Big|m^{(k)}(s,a)-m^{\star}(s,a)+\sum_{s'\in\states}\Big(p^{(k)}(s'|s,a)-p^{\star}(s'|s,a)\Big)V(s')\Big|\\&+\Big|\sum_{s'\in\states}\Big(p^{(k)}(s'|s,a)-p^{\star}(s'|s,a)\Big)\Big({V(s')}-V_{m^{(k)}}^{\boldpi_k,p^{(k)}}(s',h+1)\Big)\Big|
    \\&\leq  \widehat{\bonus}_k(s,a)+ \wsedit{2H^2}\sqrt{\frac{2S\ln\big(16SAH^2(d+1)k^2/(\epsilon \delta))}{N_k(s,a)}}+\epsilon S.
\end{align*}

Upper bounding $\widehat{\bonus}_k(s,a)\leq \wsedit{2H^2}\sqrt{\frac{2S\ln\big(16SAH^2(d+1)k^2/(\epsilon \delta))}{N_k(s,a)}}$ completes the lemma.

\tlcomment{second term needs to be $H^2$; and we need to clip the bonus; acknowledge the person who emailed Yufeng Zhang (Northwestern University).}
\end{proof}

\subsection{Final guaraantee for the basic setting (Theorem~\ref{thm:tabular})}\label{app:guarantee_basic}
\begin{proof}
The failure probability of the algorithm is $\delta$ due to the validity of bonus $\widehat{\bonus}_k(s,a)$ (Lemma~\ref{lem:valid_bonus}) and another $\delta$ by the bound on Bellman error (Lemma~\ref{lem:bellman_error_bound}). When neither failure events occur (probability $1-2\delta$), Proposition~\ref{prop:regret_decomposition} upper bounds either of reward or consumption regret by $\En^{\boldpi_k}\Big[\Big|\bellman_{m^{(k)}}^{\boldpi_k ,p^{(k)}}(\vs_h,\va_h,h)\Big|\Big]$. By  Lemma~\ref{lem:bellman_error_bound}, the Bellman error at episode $t$, for $\epsilon>0$, is at most:
\begin{align*}
   \Big|\bellman_{m^{(t)}}^{\boldpi_t ,p^{(t)}}(\vs_{t,h},\va_{t,h},h)\Big|\leq  \wsedit{4H^2}\sqrt{\frac{2S\ln\big(16SAH^2(d+1)t^2/(\epsilon \delta))}{N_t(s,a)}}
    +\epsilon S
\end{align*}
Summing across all $h=1\ldots H$ and $t=1,\ldots, k$, the sum of Bellman errors is at most:
\begin{align*}
	\sum_{t=1}^k &\sum_{h=1}^H \left\lvert \bellman_{m^{(t)}}^{\boldpi_t, p^{(t)}}(\vs_{t,h}, \va_{t,h}, h) \right\rvert \\
    & \leq \sum_{t=1}^k\sum_{h=1}^H
   \Big( \wsedit{4H^2}\sqrt{\frac{2S\ln\big(16SAH^2(d+1)t^2/(\epsilon \delta))}{N_t(s,a)}}  +\epsilon S\Big)
   \\& \leq \sum_{s,a}\Big( \sum_{j=1}^{2H} \wsedit{4H^2}\sqrt{2S\ln\big(16SAH^2(d+1)k^2/(\epsilon\delta\big)}\\ &\qquad+\sum_{j=H+1}^{N_k(s,a)} \wsedit{4H^2}\sqrt{\frac{4S\ln\big(16SAH^2(d+1)k^2/(\epsilon \delta))}{j}}
    +\epsilon S\Big)
    \end{align*}
 The second inequality follows since a particular state-action pair may have the same visitations for $H$ times (as we only update this quantity at the end of the episode). To avoid incurring an additional dependence on $H$, we separate the first $H$ visitations of each state-action pair and treat the bound as if $j=1$ for them.~\footnote{The reason why we sum until $2H$ in the first term is since we want to consider all such visitations that occur in an episode that started with $N_k(s,a)<H$; the additional factor of $2$ in the second term comes since, $j/N_{t}(s,a)\leq 2$ if $N_{t}(s,a)\geq H$ and the $j$-th visitation happens within the same episode.} For the remaining visitations, $j$ and $N_{k}(s,a)$ are always within a factor of $2$ and this factor therefore appears within the square root.

We now bound the second term:    
    \begin{align*}
    &\sum_{s,a}\Big(\sum_{j=H+1}^{N_k(s,a)} \wsedit{4H^2}\sqrt{\frac{4S\ln\big(16SAH^2(d+1)k^2/(\epsilon \delta))}{j}}
    +\epsilon S\Big)
    \\&\leq \wsedit{4S A H^2} \sqrt{N_k(s,a)\ln\big(N_k(s,a)\big)\cdot 4S\ln\big(16SAH^2(d+1)k^2/(\epsilon \delta))}+\epsilon kHS\\
    &\leq \wsedit{4SAH^2}\sqrt{\frac{kH \cdot 4S\cdot \ln(k)\ln\big(16SAH^2(d+1)k^2/(\epsilon \delta)\big)}{SA}}+\epsilon kHS\\
    &\leq \wsedit{16}S\sqrt{A\wsedit{H^5}}\cdot\sqrt{k}\cdot\sqrt{\ln(k)\ln\big(2SAH(d+1)k/\delta\big)}+1.
\end{align*}
The last inequality holds by setting $\epsilon=\frac{1}{kHS}$.

The first term can be bounded by additive terms that depend only logarithmically on $k$:
\begin{align*}&\sum_{s,a}\Big( \sum_{j=1}^{2H} \wsedit{4H^2}\sqrt{2S\ln\big(16SAH^2(d+1)k^2/(\epsilon\delta\big)}\leq \wsedit{32 S^{3/2}AH^3}\sqrt{\ln(2SAH(d+1)k/\delta\big)}
\end{align*}
As a result:
\begin{align*}\sum_{t=1}^k \sum_{h=1}^H \left\lvert \bellman_{m^{(t)} }^{\boldpi_t, p^{(t)} }(\vs_{t,h}, \va_{t,h}, h) \right\rvert &\leq \wsedit{16 S\sqrt{AH^5}}\sqrt{k}\cdot\sqrt{\ln(k)\ln\big(2SAH(d+1)k/\delta\big)}+1\\&+\wsedit{32 S^{3/2}AH^3}\sqrt{\ln\big(2SAH(d+1)k/\delta\big)}
\end{align*}

Now we link the additive Bellman error to the expected sum of Bellman errors under the expectation of the policies $\{\boldpi_t\}$ (as needed by Proposition~\ref{prop:regret_decomposition}) via a simple martingale argument. From Lemma~\ref{lemma:expectation_to_realization}, with probability at least $1-\delta$, we have:
\begin{align*}
&\left\lvert  \sum_{t=1}^k \sum_{h=1}^H \left\lvert \bellman_{m^{(t)} }^{\boldpi_t, p^ {(t)} }(\vs_{t,h}, \va_{t,h}, h) \right\rvert - \sum_{t=1}^k \sum_{h=1}^H \Exp^{\boldpi_t}\left[\sum_{h=1}^H \left\lvert \bellman_{m^{(t)}}^{\boldpi_t, p^{(t)}}(\vs_h, \va_h, h)  \right\rvert \right] \right\rvert \\ &\qquad  \qquad\qquad\qquad\qquad \qquad \leq \wsedit{5 H^{2.5}}\sqrt{2\ln(4k^2 /\delta) k},
\end{align*} where we use the fact that $\left\lvert \bellman^{\boldpi, p}_{m} \right\rvert \leq \wsedit{5 H^2}$ due to of $Q^{\boldpi, p}_{m}\sa\in [0,\wsedit{2 H^2}]$, $m^\star\sa \in [0,1]$, and $V^{\boldpi, p}_m(s) \in [0, \wsedit{2H^2}]$. 
Combining the above, we conclude the proof.
\end{proof}




\section{Analysis: concave-convex setting (Section~\ref{sec:concave-convex})}
\label{app:convex}
In this section, we prove the main guarantee for the convex-concave setting. Since the regret decomposition of the basic setting (Proposition~\ref{prop:regret_decomposition}) does not hold direclty as $f$ and $g$ are not linear, we need to create an analogous regret decomposition (Proposition~\ref{prop:decomposition_convex}) for the convex-concave setting. This can be done by leveraging the Lipschitzness of the functions. 
Armed with this new regret decomposition, we can directly call the results we have for for the basic setting (e.g., upper bounds of Bellman errors) to conclude the regret analysis for the convex-concave setting. 
The first step leading to this regret decomposition is to show that $\boldpi^\star$ is a feasible solution of $\textsc{ConvexConPlanner}$.

\subsection{Feasibility of optimal policy in concave-convex setting (Lemma \ref{lem:feasible_opt_convex})}
\label{app:feasible_mean_value}
\begin{lemma}\label{lem:feasible_opt_convex}
If the bonus $\widehat{\bonus}_k$ is valid (in the sense of Definition~\ref{defn:valid_bonus}) then policy $\boldpi^{\star}$ that maximizes the objective of the convex-concave setting is feasible in $\textsc{ConvexConPlanner}$.
\end{lemma}
\begin{proof}
Unlike the linear case, the feasibility of $\boldpi^{\star}$, requires more care. Applying the same dynamic programming arguments as in Lemma~\ref{lem:valid_implies_optimism}, it follows that:
\begin{align*}
    \forall i\in\resources:\qquad \Exp\Big[V_{\widehat{c}_{i,k}-\bonus_k}^{\boldpist,{p^{(k)}}}(\vs_1,1)\Big]\leq \Exp_{s}\Big[V_{c_i^{\star}}^{\boldpist,p^{\star}}(\vs_1,1)\Big]\leq \Exp\Big[V_{\widehat{c}_{i,k}+\bonus_k}^{\boldpist,{p^{(k)}}}(\vs_1,1)\Big].
\end{align*}
Letting $\widetilde{g}(\alpha)=\Exp\Big[V_{\widehat{c}_{i,k}+\alpha\bonus_k}^{\boldpist,{p^{(k)}}}(\vs(1),1)\Big]$, the above can be rewritten as:
\begin{align*}
    \forall i\in\resources:\qquad \widetilde{g}(-1)\leq \Exp\Big[V_{c_i^{\star}}^{\boldpist,p^{\star}}(\vs_1,1)\Big]\leq \widetilde{g}(1).
\end{align*}
Since $\widetilde{g}(\cdot)$ is the expected value over the same policy and under the same transitions, it is continuous with respect to its argument. As a result, applying mean-value theorem on each $i$ separately, there exists some $\alpha_i$ such that $\widetilde{g}(\alpha_i)=\Exp_{s}\Big[V_{c_i^{\star}}^{\boldpist,p^{\star}}(\vs_1,1)\Big]$. Due to the feasibility of $\boldpi^{\star}$ on the true transitions and consumptions, it holds that {$g\Big(\bm{\widetilde{g}}(\alpha_i)\Big)\leq 0$}. Hence, selecting estimates $\widehat{c}_{i,k}+\alpha_i\widehat{\bonus}_k$ creates a feasible solution for $\boldpi^{\star}$ under the estimated transitions of the {$\textsc{ConvexConPlanner}$} program. The final value of $\boldpi^{\star}$ at this program maximizes the objective retaining feasibility; hence the existence of one feasible selection of consumption estimates concludes the proof of the lemma.
\end{proof}

We conclude by remarking that proving optimism feasibility for the concave-convex setting in multiple-step RL setting is more challenging than that in single-step multi-arm bandit setting \cite{AgrawalDevanurEC14} since in bandits, there are no transitions. In the proof above, to show that $\pi^\star$ is feasible in $\textsc{ConvexConPlanner}$ which is defined with respect to $p^{(k)}$, we leverage the fact that $\widetilde{g}(\alpha)$ is continuous and a novel application of mean-value theorem to link $\pi^\star$'s performance in the optimistic model $\Exp\Big[V_{\widehat{c}_{i,k}+\alpha_i\bonus_k}^{\boldpist,{p^{(k)}}}(\vs_1,1)\Big]$ and $\pi^\star$'s performance under the real model $\Exp_{s}\Big[V_{c_i^{\star}}^{\boldpist,p^{\star}}(\vs_1,1)\Big]$.

\subsection{Regret decomposition for concave-convex setting}
Using the Lipschitz continuous assumption of $f$ and $g$, we can decompose the regret into a sum of Bellman errors as before,  but scaled by the Lipschitz constant this time.

\begin{prop}\label{prop:decomposition_convex} Let $L$ be the Lipschitz constant for $f$ and $g$. If $\widehat{\bonus}_k(s,a,\delta)$ is valid for all episodes $k$ simultaneously then the per-episode reward and consumption regrets can be upper bounded by:
\begin{align*}f\Big(\Exp^{\boldpist, p^{\star}}\Big[\sum_{h=1}^Hr^{\star}(\vs_h,\va_h)\Big]\Big)-f\Big(\Exp^{\boldpi_k,p^{\star}}\Big[\sum_{h=1}^Hr^{\star}(\vs_h,\va_h)\Big]\Big)
\leq L \cdot\Exp^{\boldpi_k}\Big[\sum_{h=1}^H \bellman_{r^{(k)}}^{\boldpi_k,p^{(k)}}\big(\vs_h,\va_h,h)\Big]\Big)\\
g\Big(\Exp^{\boldpi_k, p^{\star}}\Big[\sum_{h=1}^H\bm{c}^{\star}(\vs_h,\va_h,i)\Big]\Big)
\leq L\sum_{i\in\resources}\cdot\Exp^{\boldpi_k}\Big[\sum_{h=1}^H \Big|\bellman_{c_i^{(k)}}^{\boldpi_k,p^{(k)}}(\vs_h,\va_h,h)\Big|\Big]
\end{align*}
\end{prop}
\begin{proof}
We first prove the reward requirement. Let $r(\boldpi)$ be the solution of the inner maximization program for policy $\boldpi$, and we define $r^{(k)}=r(\boldpi_{k})$. For notational convenience, we denote $V_m^{\boldpi,p}=\Exp^{\boldpi,p}\Big[V_m^{\boldpi,p}\Big]$ Since $r^{\star}(s,a)\in[\widehat{r}(s,a)-\widehat{\bonus}_k(s,a,\delta),\widehat{r}(s,a)+\widehat{\bonus}_k(s,a,\delta)]$ and the bonus $\widehat{\bonus}_k$ is valid, similar to Lemma~\ref{lem:valid_implies_optimism}, it holds: \begin{align}V_{r^{\star}}^{\boldpi^{\star},p^{\star}}\in \Big[V_{\widehat{r}-\bonus}^{\boldpi^{\star},{p^{(k)}}},V_{\widehat{r}+\bonus}^{\boldpi^{\star},{p^{(k)}}}\Big].\label{eq:mean_value_convex}\end{align}
As a result, by mean-value theorem, there exists $\alpha\in[-1,1]$ such that $V_{r^{\star}}^{\boldpi^{\star},p^{\star}}=V_{\widehat{r}+\alpha\bonus}^{\boldpi^{\star},{p^{(k)}}}$.
Since $\boldpi_k$ is the maximizer of $\textsc{ConvexConPlanner}$ and $\boldpi^{\star}$ is feasible for that program, it holds that:
\begin{align}
f\Big(V_{r(\boldpi_k)}^{\boldpi_k,{p^{(k)}}} \Big)&\geq f\Big(V_{r(\boldpi^{\star})}^{\boldpi^{\star},{p^{(k)}}}\Big)\geq f\Big(V_{\widehat{r}+\alpha\bonus}^{\boldpi^{\star},{p^{(k)}}}\Big){=} f\Big(V_{r^{\star}}^{\boldpi^{\star},p^{\star}}\Big),\label{eq:convex_optimism}
\end{align}
where the second-to-last inequality holds since $r(\boldpi^{\star})$ is the maximizer of the inner program for $\boldpi^{\star}$ and the equality holds by \eqref{eq:mean_value_convex}.

We are now ready to provide the equivalent of the regret decomposition:
\begin{align*}
    f(V_{r^{\star}}^{\boldpi^{\star},p^{\star}})-f(V_{r^{\star}}^{\boldpi_k,p^{\star}})&\leq f(V_{r(\boldpi_k)}^{\boldpi_k,{p^{(k)}}})-f(V_{r^{\star}}^{\boldpi_k,p^{\star}})\leq  L\cdot\Big|V_{r(\boldpi_k)}^{\boldpi_k,{p^{(k)}}}-V_{r^{\star}}^{\boldpi_k,p^{\star}}\Big|
    \\&\leq L\cdot \Exp^{\boldpi_k}\left({\sum_{h=1}^H \bellman_{r^{(k)}}^{\boldpi_k,p^{(k)}}\big(\vs_h,\va_h,h\big)}\right)
\end{align*}
where the first inequality holds by \eqref{eq:convex_optimism}. the second inequality by Lipschitzness and the last inequality holds by simulation lemma (Lemma~\ref{lem:simulation}). 

For the consumption requirement, since $\boldpi_k$ is feasible in $\textsc{ConvexConPlanner}$, denoting again by $\bm{c}(\boldpi)$ the consumption in the maximizer for policy $\boldpi$ in the inner mathematical program. Same as above we define $\bm{c}^{(k)}=\bm{c}(\boldpi_k)$. It holds that: 
\begin{align}g\Big(\Exp^{\boldpi_k,{p^{(k)}}}\Big[\sum_{h=1}^H\bm{c}_h(\boldpi_k)\Big]\Big)\leq 0 \
\end{align}
As a result,  
\begin{align*}
    g\Big(\Exp^{\boldpi_k,p^{\star}}\Big[\sum_{h=1}^H\bm{c}^{\star}_h\Big]\Big)-g\Big(\Exp^{\boldpi_k,{p^{(k)}}}\Big[\sum_{h=1}^H\bm{c}_h(\boldpi_k)\Big]\Big)&\leq L\left\|\Exp^{\boldpi_k,p^{\star}}\Big[\sum_{h=1}^H\bm{c}^{\star}_h\Big]-\Exp^{\boldpi_k,{p^{(k)}}}\Big[\sum_{h=1}^H\bm{c}_h(\boldpi_k)\Big]\right\|_1\\
    &{=} L\sum_{i\in \resources} \Big|\Exp^{\boldpi_k,p^{\star}}\Big[\sum_{h=1}^Hc^{\star}_h(i)
    \Big]-\Exp^{\boldpi_k,{p^{(k)}}}\Big[\sum_{h=1}^H c_{h}(\boldpi_k,i)
    \Big]
    \Big|\\
    &\leq L\cdot \sum_{i\in\resources}  \Exp^{\boldpi}\left({\sum_{h=1}^H \Big|\bellman_{c_i^{(k)}}^{\boldpi_k,p^{(k)}}\big(\vs_h,\va_h,h\big)\Big|} \right),
\end{align*} where again we applied Lipschitness and simulation lemma. 
\end{proof}

\subsection{Concave-convex theorem (Theorem \ref{thm:convex})}

\begin{proof}[Proof of Theorem~\ref{thm:convex}]
The proof follows similarly to the proof of Theorem~\ref{thm:tabular} by replacing Proposition~\ref{prop:regret_decomposition} with Proposition~\ref{prop:decomposition_convex}. The linear dependency on $d$ in the consumption regret comes from the fact that the Lipschitzness of $g$ is defined in L1 norm. 
\end{proof}

\section{Analysis: Knapsack setting (Section~\ref{sec:knapsacks})}\label{app:knapsacks}
In this section, we prove the guarantee for the hard-constraint setting. The goal is to show that over $K$ episodes, our algorithm has sublinear reward regret comparing to the best dynamic policy (formally defined in Appendix~\ref{app:dynamic}), while satisfying hard budget constraints with high probability.
\subsection{Theorem with hard constraints (Theorem~\ref{thm:knapsacks})}
\begin{proof}[Proof of Theorem~\ref{thm:knapsacks}] We denote by $\textsc{Opt}$ the expected total reward of $\boldpi^\star$. 
Consider now the policy $\widetilde{\pi}^\star$ that selects the null policy with probability $\epsilon$ and follows $\boldpi^{\star}$ otherwise. This policy is feasible for \eqref{eq:approx_program_knapsack}; as a result the expected reward  $\widetilde{\boldpi}^{\star}$ for \eqref{eq:approx_program_knapsack} is at least $(1-\epsilon)\opt$. 
Since the total reward is upper bounded by $KH$, it therefore holds that:
\begin{align}\label{eq:approx_to_dynamic_knapsacks}
    \sum_{k=1}^K\Exp^{{\widetilde{\boldpi}}^{\star}}\Big[\sum_{h=1}^H r^{\star}\big(\vs_h,\va_h\big)\Big]\geq (1-\epsilon)\opt\geq \opt- \epsilon KH
\end{align}
In the high-probability event where the regret guarantee of $\textsc{AggReg}(\delta)$ does not fail, the reward of the algorithm is at least:
\begin{align}\label{eq:algorithm_to_approx_knapsacks}
    \sum_{k=1}^K \sum_{h=1}^H r_{k,h}\geq  \sum_{k=1}^K\Exp^{{\widetilde{\boldpi}}^{\star}}\Big[\sum_{h=1}^H r^{\star}(\vs_h,\va_h)\Big] - \textsc{AggReg}(\delta),
\end{align} 
Combining \eqref{eq:approx_to_dynamic_knapsacks} and \eqref{eq:algorithm_to_approx_knapsacks}, with probability $1-\delta$, the reward regret with respect to $\boldpi^\star$ is at most:
\begin{align}\label{eq:reward_regret_knapsack}
    \textsc{RewReg(K)} \leq \frac{1}{K}\textsc{AggReg}(\delta)+\epsilon H
\end{align}

We now focus on the consumption. Since we optimize \eqref{eq:approx_program_knapsack}, for any resource $i\in\resources$, when the regret guarantee  $\textsc{AggReg}(\delta)$  against $\widetilde{\boldpi}^\star$ does not fail and given that ${\widetilde{\boldpi}}^{\star}$ is feasible for \eqref{eq:approx_program_knapsack}, it holds that:
\begin{align*}
    \sum_{k=1}^K\sum_{h=1}^H c_{k,h,i}\leq \sum_{k=1}^K\Exp^{{\widetilde{\boldpi}}^{\star}}\Big[\sum_{h=1}^H c\big(\vs_h,\va_h,i\big)\Big]+\textsc{AggReg}(\delta)\leq (1-\epsilon)B_i+\textsc{AggReg}(\delta)
\end{align*}
Hence, when the regret guarantee $\textsc{AggReg}(\delta)$ does not fail, the consumption is less than $B_i$ for all $i$ as long as $\epsilon\geq \frac{\textsc{AggReg}(\delta)}{\min_i B_i}$. Moreover $\epsilon$ is a probability as a result it should also be less than $1$ which holds when $\min_i B_i\geq \textsc{AggReg}(\delta)$.
Applying on \eqref{eq:reward_regret_knapsack} and assuming without loss of generality that $KH>\min_i B_i$ (otherwise the setting is essentially unconstrained), the reward regret is at most $$ \textsc{RewReg}(K)\leq \frac{2H\textsc{AggReg}(\delta)}{\min_i B_i}.$$
\end{proof}

\subsection{Dynamic policy benchmark}
\label{app:dynamic}
We call a policy \emph{dynamic} if it 
maps the entire history to a distribution over the action space. Specifically we denote history $\mathcal{H}_{k,h}$ as the history that contains all the information from the beginning of the first episode to the end of the step $h-1$ at the $k$-th episode plus the state at step $h$ in episode k. At any episode k and step $h$, a dynamic policy $\widetilde{\boldpi}(\cdot | \mathcal{H}\kh) \in \Delta(\mathcal{A})$  maps history $\mathcal{H}\kh$ to a distribution over action space. We denote ${\Pi}_{\text{dynamic}}$ as the set of all dynamic policies that satisfies the budget constraints deterministically, i.e., for any $\widetilde{\boldpi}\in {\Pi}_{\text{dynamic}}$, when executed for $K$ episodes in the MDP, we have $\sum_{k=1}^K \sum_{h=1}^H c_i(\vs_{k,h}, \va_{k,h}) \leq B_i$ for all $i\in\resources$, deterministically. Ideally we want to compare against the best dynamic policy that maximizes the expected total reward $\max_{\widetilde\boldpi\in{\Pi}_{\text{dynamic}}}\Exp^{\widetilde\boldpi}\left[\sum_{k=1}^K \sum_{h=1}^K r_{k,h}\right]$. We denote such an optimal dynamic policy as $\widetilde\boldpi^\star$ and its expected total reward across K episodes as 
\begin{align*}
\textsc{Opt} :=\max_{\widetilde\boldpi\in{\Pi}_{\text{dynamic}}}\Exp^{\widetilde\boldpi}\left[\sum_{k=1}^K \sum_{h=1}^K r_{k,h}\right].
\end{align*}

The lemma below shows that indeed the stationary Markovian policy $\boldpi^\star$ actually achieves no smaller expected total reward across K episodes than that of the best dynamic policy.

\begin{lemma}\label{lem:optimal_dynamic_policy_knapsacks}
The reward of the policy $\boldpi^{\star}$ maximizing program \eqref{eq:objective} with $\xi(i)=\frac{B_i}{K}$ is at least as large as the per-episode reward of the optimal dynamic policy that is subject to hard constraints instead:
\begin{align*} \Exp^{\boldpi^\star}\Big[\sum_{h=1}^Hr^{\star}\big(\vs_h,\va_h\big)\Big] \geq \frac{1}{K}\max_{\widetilde\boldpi\in\Pi_{\text{dynamic}}}\Exp^{\widetilde\boldpi}\Big[ \sum_{k=1}^K \sum_{h=1}^{H} r(\vs_{k,h}, \va_{k,h})\Big] = \frac{\textsc{Opt}}{K}.
\end{align*}
\end{lemma}
\begin{proof}Denote $\widetilde\boldpi^\star$ as the optimal dynamic policy from $\Pi_{\text{dynamic}}$. Any policy induces a state-action distribution at episode $k$ and stage $h$, denoted as $\rho_{\widetilde\boldpi}(s, a; h, k)$, which stands for the probability of $\widetilde\boldpi$ visits state-action pair $\sa$ at stage $h$ in episode $k$.  Denote $\rho_{\widetilde\boldpi}(s,a ; h) = \sum_{k=1}^K \rho_{\widetilde\boldpi}(s, a; h, k) / K$ which stands for the probability of $\widetilde\boldpi$ visiting $\sa$ at stage $h$. We have:
\begin{align*}
\sum_{a} \rho_{\widetilde\boldpi}(s', a; h, k) = \sum_{s,a} \rho_{\widetilde\boldpi}(s, a; h-1, k) p^{\star}(s' | s,a), \forall s',
\end{align*} due to the Markovian transition $p^{\star}(s' | s,a)$, which implies that:
\begin{align*}
\sum_{a} \rho_{\widetilde\boldpi}(s', a; h) = \sum_{s,a} \rho_{\widetilde\boldpi}(s, a; h-1) p^{\star}(s' | s,a), \forall s'. 
\end{align*} Hence, $\rho_{\widetilde\pi}(s,a; h)$ satisfies the flow constraints, and hence induces a stationary Markovian policy:
\begin{align*}
\boldpi_{\widetilde\boldpi}(a |s) \propto \rho_{\widetilde\boldpi}(s,a; h) / \sum_{a}\rho_{\widetilde\boldpi}(s,a;h),
\end{align*} and $\boldpi_{\widetilde\boldpi}$ induces state-action visitation distribution that are exactly equal to  $\rho_{\widetilde\boldpi}(s,a; h)$. 

Note that $\widetilde\boldpi^\star$ satisfies the budget constraints deterministically, which means in expectation, it will satisfies the constraints as well, i.e., 
\begin{align*}
\sum_{k=1}^K \sum_{h=1}^H \sum_{\sa} \rho_{\widetilde\boldpi^\star}(s,a; h) c_i(s,a) \leq B_i, \quad \forall i\in\resources,
\end{align*} which implies that in expectation, for $\boldpi_{\widetilde\boldpi^\star}$, we have that for all $i\in\resources$:
\begin{align*}
\Exp^{\boldpi_{\widetilde\boldpi^\star}}\left[\sum_{h=1}^H c_i(s_h,a_h)\right] = \sum_{h=1}^H \sum_{\sa}\rho_{\boldpi_{\widetilde\boldpi^\star}}(s,a,h) c_i(s,a) =\sum_{k=1}^K \sum_{h=1}^H \sum_{\sa} \rho_{\widetilde\boldpi^\star}(s,a; h) c_i(s,a) / K \leq B_i/K.
\end{align*}This means that $\boldpi_{\widetilde\boldpi^\star}$ is a feasible solution of the hard-constraint program.

Similarly, we have that the expected per-episode total reward of $\widetilde\boldpi^\star$ is the same as the expected total reward of $\boldpi_{\widetilde\boldpi^\star}$:
\begin{align*}
\Exp^{\boldpi_{\widetilde\boldpi^\star}}\Big[ \sum_{h=1}^H r_h(s_h,a_h)\Big] = \frac{1}{K} \Exp^{\widetilde\boldpi^\star}\left[\sum_{k=1}^K \sum_{h=1}^H r_{k,h} \right].
\end{align*} 

Hence, due to the optimality of $\boldpi^\star$, we immediately have:
\begin{align*}
\Exp^{\boldpi^\star}\Big[\sum_{h=1}^H r_h\Big] \geq \Exp^{\boldpi_{\widetilde\boldpi^\star}}\Big[\sum_{h=1}^H r_h\Big] = \frac{1}{K} \Exp^{\widetilde\boldpi^\star}\Big[\sum_{k=1}^K \sum_{h=1}^H r_{k,h} \Big].
\end{align*}
\end{proof}

Since our approach incurs sublinear regret with respect to $\boldpi^\star$, it follows from the above lemma that it incurs sublinear regret with respect to $\textsc{Opt}$ -- the total reward across $K$ episodes from the best dynamic policy.

\section{Experimental details} \label{app:experiment}
 In the experiments, both $\textsc{ApproPO}$ and $\textsc{RCPO}$ use the same policy gradient algorithm, specifically, Advantage Actor-Critic (A2C) \cite{mnih2016asynchronous} as the learning algorithm. We implemented $\textsc{ConRL}$ using two version of \textsc{LagrConPlanner} (see algorithm~\ref{alg:lagrange_conplanner} below) as \textsc{ConPlanner} in which the planner is either value iteration (exact planner) or A2C (approximate planner similar to Dyna model-base RL \cite{dyna1991stutton}) using fictitious samples. All three algorithms have outer-loop learning rates which we tuned while hyperparameters used for A2C is same across all three methods. Here, we report the result for the best learning rate for each method.
 
\subsection{\textsc{LagrConPlanner}}
\label{app:lagr_conplanner}
Our theoretical results posit that $\textsc{ConPlanner}$ is solved optimally, which can be indeed achieved via linear programming (see \autoref{app:algorithm_details}). However in our experiments it suffices to 
use a general heuristic for $\textsc{ConPlanner}$. Our approach is to Lagrangify the constraints, and create a min-max mathematical program with the Lagrangean objective:
\begin{align*}
&\min_{\forall i\in\resources:\lambda(i)\leq 0} \max_{\boldpi}\Big(\Exp^{\boldpi,p^{(k)}}\Big[\sum_{h=1}^H r^{(k)}\big(\vs_h,\va_h\big)\Big]  +\sum_{i\in\resources}\lambda(i)\Big(\Exp^{\boldpi, p^{(k)}}\Big[\sum_{h=1}^H c^{(k)}\big(\vs_h,\va_h, i\big)\Big]- \xi(i)\Big).
\end{align*}
Define pseudo-reward $r_\lambda^{(k)}$ as 
\begin{align*}
    r_\lambda^{(k)}(s,a) = r^{(k)}(s,a) + \sum_{i \in D} \lambda(i)[c^{(k)}(s,a)-\xi(i)]
\end{align*}
With a fixed choice of Lagrange multipliers $\{\lambda(i)\}_{i\in\resources}$, this is an unconstrained \emph{planning} program which we refer to as $\textsc{Planner}(p^{(k)},r_\lambda^{(k)})$ and it can be solved by a planning oracle. 

We update Lagrange multipliers via projected gradient descent \cite{Zinkevich03}. The overhead of \textsc{ConPlanner} is  computational, as we do not require new samples. The full procedure is in Algorithm \ref{alg:lagrange_conplanner}.  The near-optimality of Algorithm~\ref{alg:lagrange_conplanner} can be proved by leveraging the fact that we are iteratively updating $\pi$ and $\lambda$ using no-regret online learning procedure (Best Response for $\pi$ and OGD for $\lambda$) (e.g., \cite{cesa2006prediction}). We omit the analysis for Algorithm~\ref{alg:lagrange_conplanner} as it is not the main focus of this work. 

\begin{algorithm}[ht]
 \caption{Lagrangean-based Constrained Planner (\textsc{LagrConPlanner})}
 \begin{algorithmic}[1]
 \label{alg:lagrange_conplanner}
 \STATE \textbf{hyper-parameters:} {learning rate $\eta$} 
 \STATE \textbf{Input:} Estimates $\widehat{p}_k$,  $\widehat{r}_k$, $\widehat{\bm{c}}_k$ and bonus $\hat{b}_k$
 \STATE \textbf{Compute bonus-enhanced model} $\model^{(k)}=\big(p^{(k)},r^{(k)},\bm{c}^{(k)})$\\
      $$p^{(k)}(s' | s,a) = \widehat{p}_k(s' | s,a ) \quad \forall s,a,s'$$\\
      $$r^{(k)}(s,a) = \widehat{r}_k(s,a)+\widehat{b}_k(s,a)\quad  \forall s,a $$\\
      $$c^{(k)}(s,a,i) = \widehat{c}_k(s,a,i)-\widehat{b}_k(s,a)\quad \forall s,a , i \in \resources$$
 \STATE Initialize Lagrange parameters $\lambda_1(i) \gets 0$ for $i\in \resources$
 \FOR{Iteration $k$ from $1$ to $N$ }
  \STATE Define 
  $$
    r_\lambda^{(k)}(s,a) = r^{(k)}(s,a) + \sum_{i \in D} \lambda(i)[c^{(k)}(s,a)-\xi(i)]
  $$
  \STATE $\pi_k \gets \textsc{Planner}(p^{(k)},r_\lambda^{(k)})$\
  \STATE $ \lambda_{k+1}(i) \gets \min\left\{0,  \lambda_k(i) - \eta  \Exp^{\pi_k,p^{(k)}}\left[\sum_{h=1}^H [c^{(k)}(\vs_h,\va_h,i)] - \xi(i)\right] \right\}\quad \forall i\in\resources$
 \ENDFOR
 \STATE \textbf{Return} mixture policy $\boldpi := \frac{1}{N} \sum_{k=1}^N \pi_k$\;
 \end{algorithmic}
\end{algorithm}
In our experiments, two versions of $\textsc{Planner}$ have been implemented: Value Iteration (exact planner) and A2C with fictitious samples (approximate planner)
\paragraph{Value Iteration as \textsc{Planner}} This program takes $p$ and $r$ as input. Finite horizon value iteration is simply solving the following acyclic dynamic program. 
$$
    Q(s,a,h)=\begin{cases}
    0 & h=H+1\\
    r(s,a) + \sum_{s'}\big[p(s'|s,a)\max_{a'}Q(s',a',h+1)\big] &h=1,\dots,H
    \end{cases}
$$
then the optimal policy for step $h$ is computed as
$$
    \boldpi_h(s)=\mathrm{argmax}_{a} Q(s,a,h)
$$
and the algorithm returns the $H$-step policy
$$
    \boldpi = (\boldpi)_{h=1}^{H}
$$
\paragraph{\textsc{A2C} with fictitious samples as \textsc{Planner}} \smedit{This program takes $p$ and $r$ as input, then, using model $p$ and $r$ it generates episodes and use those samples to train our A2C agent. Since we only call this subroutine with our estimated model ($p \leftarrow \hat{p}$ and $r \leftarrow \hat{r}$) those episodes are fictitious (not adding to sample complexity). The algorithm is given Algorithm~\ref{alg:a2c_fict} (Parameterized policy $\boldpi_\theta$ and value function estimate $V_\theta$)}
\begin{algorithm}[ht]
 \caption{A2C planner with fictitious samples}
 \begin{algorithmic}[1]
 \label{alg:a2c_fict}
 \STATE \textbf{hyper-parameters:} {learning rate $\eta$, $\alpha \in [0,1]$} 
 \STATE \textbf{Input:} transitions $p$,  reward function $r$
 \STATE Define A2C loss
 $$
   L(\theta) = \Exp^{\boldpi_\theta,p}[\sum_{h=1}^{H} -\log \boldpi_\theta(a_h|s_h)(R(h)-V_\theta(s_h)) +\alpha(R(h)-V_\theta(s_h))^2]
   $$
   $$
    R(h)=\sum_{h'=h}^{H} r(s_h,a_h)
   $$
 \STATE Initialize $\theta$ arbitrarily
 \FOR{Iteration $i$ from $1$ to $T$ }
  \STATE Emulate an episode by running $\boldpi_\theta$ on MDP with transitions $p$ and reward function $r$
  \STATE update $\theta \leftarrow \theta - \eta \nabla_\theta L(\theta)$
 \ENDFOR
 \STATE \textbf{Return} $\boldpi_\theta$
 \end{algorithmic}
\end{algorithm}

\subsection{Hyperparameter Tuning}
Both $\textsc{ConRL-A2C}$ and $\textsc{RCPO}$ used the Adam optimizer. For our method we performed a hyperparamter search on both domains over the following values in Table \ref{hyperparam_consider}; selected values are given in Table~\ref{hyperparam_select}. Note that reset row refers to when using the A2C planner during each call to the planner we tried the following options: (warm-start) reuse previous weights and reset the optimizer (warm -start), or (continue) continue learning using the previous weights (continue) and optimizer, or (none) reset the model weights and optimizer.

\begin{table}[ht]
  \caption{Considered Hyperparameters}
  \centering
  \label{hyperparam_consider}
  \begin{tabular}{ll}
    Hyperparameter & Values Considered \\
    \hline
    A2C learning rate & $10^{-2},10^{-3},10^{-4}$ \\
    lambda learning rate& $10^{0},\{1,2,5\} \times 10^{-1},2\times 10^{-2},10^{-3},2\times 10^{-3}$ \\
    reset & warm-start, continue, none  \\
    conplanner iterations & $10,20,30,50,100,150, 200, 250$ \\
    A2C Entropy coeff & $10^{-3}$ \\
    A2C Value loss coeff & $0.5$ \\
  \end{tabular}
\end{table}

\begin{table}[ht]
  \caption{Selected Hyperparameters}
  \centering
  \label{hyperparam_select}
  \begin{tabular}{lll}
    Hyperparameter &  Gridworld  & Box\\
    \hline
    A2C learning rate  & $10^{-3}$ & $10^{-3}$\\
    lambda learning rate & $2\times 10^{-1}$ & $10^{-2}$ \\
    reset & none & none \\
    conplanner iterations & $10$ & $10$\\
    A2C Entropy coeff  & $10^{-3}$ & $10^{-3}$\\
    A2C Value loss coeff & $0.5$ & $0.5$\\
  \end{tabular}
\end{table}

\smedit{\subsubsection{\textsc{TFW-UCRL2}}
We used the code provided by the author (with no algorithmic parameter changed). Moreover, \textsc{TFW-UCRL2} uses weights $(L_0,L_1,\dots,L_k)$ in the objective function $g(w)$ defined in Equation 1 in \cite{Cheung19}. We only tuned these weights to identify the one maximizing the reward while guaranteeing the constraint satisfaction (for a more fair comparison with the baseline). In our experiments, we have $k=2$ and you can see the performance of \textsc{TFW-UCRL2} for $L_0=1$ and $L_1 \in \{10^{-2},10^{-3},10^{-4},10^{-5}\}$ in Figure~\ref{fig:tuning_tfw}.}

\begin{figure*}[ht]
\centering
    \includegraphics[width=0.9\textwidth]{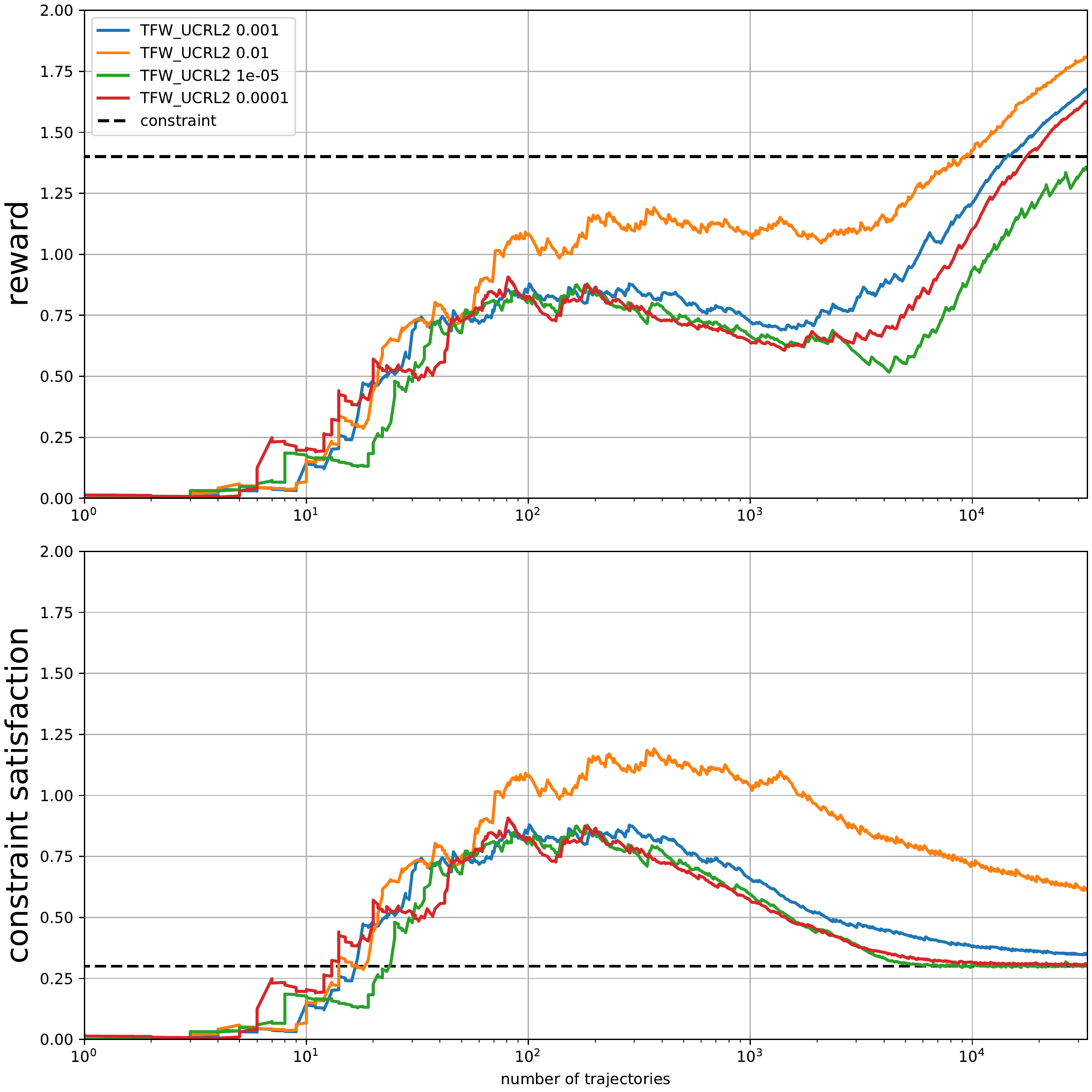}
    \caption{Performance of \textsc{TFW-UCRL2} with different choices of $L_1$ ($L_0=1$)}
    \label{fig:tuning_tfw}
    \vspace{-10pt}
\end{figure*}

\section{Concentration tools}\label{app:auxil}

This section contains general concentration inequalities that are not tied with the constrained RL setting considered in the paper.

\newcommand{\calF}{\mathcal{F}}

\begin{lemma}[Hoeffding]\label{lem:azuma_hoeffding}
    \vspace{.1in}
    Let $\{X_i\}_{i=1}^N$ be a set with each $X_i$ i.i.d sampled from some distribution and $\Exp[X_i] = 0$ for all $i$ and $\max_i | X_i| \leq b$. Then with probability at least $1-\delta$, it holds that:
        \begin{align*}
            \left\lvert \frac{1}{N}\sum_{i=1}^N X_i \right\rvert \leq b\sqrt{\frac{2\ln(2/\delta)}{N}}.
        \end{align*}
        \end{lemma}

    \begin{lemma}[Anytime version of Hoeffding]\label{lem:azuma_hoeffding_anytime}
    \vspace{.1in}
    Let $\{X_i\}_{i=1}^{\infty}$ be a set with each $X_i$ i.i.d sampled from some distribution and $\Exp[X_i] =0$ for all i and $\max_i |X_i |\leq b$. Then
    with probability at least $1-\delta$, for any $N\in \mathbb{N}^+$, it holds that:
        \begin{align*}
            \left\lvert \frac{1}{N}\sum_{i=1}^N X_i \right\rvert \leq b\sqrt{\frac{2\ln(4 N^2/\delta)}{N}}.
        \end{align*}
    \end{lemma}
    \begin{proof}
    We first fix $N\in\mathbb{N}^+$ and apply standard Hoeffding (Lemma~\ref{lem:azuma_hoeffding}) with a failure probability $\delta/N^2$. Then we apply a union bound over  $\mathbb{N}^+$ and use the fact that $\sum_{N > 0} \frac{\delta}{2N^2} \leq \delta$ to conclude the lemma. 
    \end{proof}

The following lemma is used when bounding the final regret in the above analysis where we bound the difference between the cumulative Bellman error along the empirical trajectories and the cumulative Bellman error under the expectation of trajectories (the expectation is taken with respect to the policies generating these trajectories cross episodes).

\begin{lemma} Consider a sequence of episodes $k = 1$ to $K$, a sequence of policies $\{\boldpi_k\}_{k=1}^K$, and a sequence of functions $\{f_k\}_{k=1}^K$ with corresponding filtration $\{\mathcal{F}_k\}$ with $\boldpi_k\in\mathcal{F}_{k-1}$ and $f_k\in\mathcal{F}_{k-1}$. Each policy $\boldpi_k$ generates a sequence of trajectory $\{s\kh, s\kh\}_{h=1}^H$. Denote a function $f_k:\mathcal{S}\times\mathcal{A}\to [0, C]$, with $f_k\in\mathcal{F}_{k-1}$. With probability at least $1-\delta$, for any $K$, we have:
\begin{align*}
\left\lvert \sum_{i=1}^K \sum_{h=1}^H f_k(s\kh,a\kh) - \sum_{k=1}^K \Exp^{\boldpi_k}\left( \sum_{h=1}^H f_k(\vs(h), \va(h)) \right) \right\rvert \leq C \sqrt{ 2\ln(4K^2 /\delta) K H }.
\end{align*}
\label{lemma:expectation_to_realization}
\end{lemma}
\begin{proof}
Denote the random variable $v\kh =  f_k\left( s\kh, a\kh \right)$. Denote $\Exp{\kh}$ as the conditional expectation that is conditioned on all history from the beginning to time step $h$ (not including step $h$) at episode $k$. Note that we have: $\Exp\kh\left[ v_k \right] = \Exp^{\boldpi_k}\left(f_k\left(s\kh, a\kh \right)\right)$. Note that $| v\kh | \leq C$ for any $k,h$ by the assumption on $f_k$. Hence, $\{v\kh\}_{k,h}$ forms a sequence of Martingales. Applying Hoeffding's inequality, we have with probability at least $1-\delta$,
\begin{align*}
\left\lvert \sum_{k=1}^K\sum_{h=1}^H v\kh - \sum_{k=1}^K \Exp^{\boldpi_k}\left( \sum_{h=1}^H f_k(\vs(h), \va(h)) \right) \right\rvert \leq C \sqrt{ 2\ln(2/\delta) KH} = C\sqrt{2\ln(2/\delta)H K}.
\end{align*}Assigning failure probability $\delta/k^2$ for each episode $k$ and using a union bound over all episodes conclude the proof. 
\end{proof}






\end{document}